\documentclass{article}

\usepackage[main, final]{neurips_2026}
\usepackage{listings}
\usepackage[utf8]{inputenc} % allow utf-8 input
\usepackage[T1]{fontenc}    % use 8-bit T1 fonts
\usepackage{hyperref}       % hyperlinks
\usepackage{url}            % simple URL typesetting
\usepackage{booktabs}       % professional-quality tables
\usepackage{amsfonts}       % blackboard math symbols
\usepackage{nicefrac}       % compact symbols for 1/2, etc.
\usepackage{microtype}      % microtypography
\usepackage{xcolor}         % colors
\usepackage{amsmath}
\usepackage{amsthm}
\newtheorem{assumption}{Assumption}
\newtheorem{definition}{Definition}
\newtheorem{lemma}{Lemma}
\newtheorem{remark}{Remark}
\newtheorem{proposition}{Proposition}

\usepackage{enumitem}
\usepackage{graphicx}
\usepackage{subcaption} 
\usepackage{float}
\usepackage{tabularx} 
\usepackage{multirow}
\usepackage{array}
\usepackage{tabularx}
\title{TagBO: LLM-Driven Task-Aware Graph Bayesian Optimization for Scientific Discovery}

\author{%
Xinzhe Yuan$^{2,1,*}$ \quad
Zimu Zeng$^{3,7,*}$ \quad
Zhuo Chen$^{4,1,*}$ \quad
Jianshu Zhang$^{4}$ \quad
Jinzong Dong$^{5}$ \\
Jingyi Chen$^{6}$ \quad
Nanyang Ye$^{4,\dagger}$ \quad
Huan Xiong$^{2,\dagger}$ \quad
Qinying Gu$^{1,\dagger}$ \\[3pt]
$^{1}$Shanghai Artificial Intelligence Laboratory \\
$^{2}$Harbin Institute of Technology \\
$^{3}$Fudan University\\
$^{4}$Shanghai Jiao Tong University \\
$^{5}$Central South University \\
$^{6}$Huazhong Agricultural University \\
$^{7}$Shenzhen Loop Area Institute\\
$^{*}$Equal contribution. \qquad
$^{\dagger}$Corresponding authors. \\[2pt]
}

\begin{document}

\maketitle

\begin{abstract}
Mixed-variable optimization in scientific discovery often requires exploring large combinatorial spaces under costly experimental evaluations, making sample-efficient Bayesian optimization (BO) essential. However, existing BO methods typically rely on structural descriptors to define similarity among categorical variables, implicitly assuming that structural similarity reflects similarity in task performance. In scientific optimization, this assumption is often violated, leading to surrogate representations that are poorly aligned with the optimization objective. To address this limitation, we propose \textbf{TagBO} (\textbf{T}ask-\textbf{a}ware \textbf{G}raph \textbf{B}ayesian \textbf{O}ptimization), a method that integrates LLM-derived relational priors with descriptor-based structural information for mixed-variable Bayesian optimization. Rather than directly using LLM outputs as predictive scores, TagBO treats LLM outputs as a potentially noisy source of task-aware relational priors and extracts coarse ordinal relational signals from LLM reasoning to construct task-aware similarity structures for scientific optimization. TagBO then performs constrained Gaussian process Bayesian optimization over the resulting joint discrete-continuous representation space. Experiments on diverse chemistry and materials formulation tasks demonstrate that TagBO consistently improves sample efficiency under fixed evaluation budgets across LLMs of varying capabilities.
\end{abstract}

\section{Introduction}
Accelerating scientific discovery through intelligent experimental design has become a central goal in AI for Science \citep{ghaghafarollahi2409sciagents,xin2025towards}. In many scientific optimization problems, practitioners must jointly select categorical variables (e.g., reagent types, reaction pathways) and tune continuous parameters (e.g., concentrations, temperatures), with each evaluation typically requiring costly real-world experiments, resulting in mixed-variable problem in which
each experimental evaluation is typically costly and
time-consuming \citep{butler2018machine,frazier2018tutorial,tom2024self,zhang2020bayesian}.
Bayesian optimization (BO) has emerged as a principled strategy
for such problems, and a central challenge it must address is
how to measure similarity among categorical choices so that the
surrogate model can generalize across the combinatorial space
\citep{baptista2018bayesian,deshwal2021bayesian,oh2019combinatorial}.

A rich body of work has tackled this challenge using domain-specific descriptors—molecular fingerprints, chemical group encodings, learned embeddings, and other structural representations \citep{gilmer2017neural,rogers2010extended,wang2016bayesian}. These methods have been widely adopted and are effective in many settings. \textbf{Yet} they often rely on a task-agnostic notion of proximity, assuming that nearby candidates in descriptor space exhibit similar task behavior—an assumption that can fail in the presence of task-specific mechanisms or activity cliffs \citep{medina2012scanning,stumpfe2019evolving}. This assumption implicitly defines the geometry over which BO generalizes. When it fails, the surrogate model may extrapolate along directions that are irrelevant, or even misleading, for the target task. For example, two solvents with nearly identical molecular structures may exhibit drastically different electrochemical stability because the mechanisms governing task performance are not captured by global structural descriptors.

These observations suggest that the challenge is not merely to design better descriptors, but to enable task-aware generalization in Bayesian optimization. In mixed-variable BO, the surrogate model generalizes according to assumptions induced by descriptor representations. When these assumptions fail to reflect the functional relationships relevant to the target task, BO may generalize along misleading directions.

Interestingly, scientific domains already contain extensive task-specific relational knowledge in the literature, including experimental findings, empirical heuristics, and mechanistic insights. Large language models (LLMs), pre-trained on broad scientific corpora, can implicitly encode how different materials or components relate under specific task contexts \citep{bran2023chemcrow,castro2023large,taylor2022galactica}. \textbf{However}, directly using raw LLM outputs to define optimization spaces remains challenging due to noise and inconsistency across candidates. At the same time, structural descriptors alone often fail to capture task-specific functional relationships. These observations suggest that neither source alone is sufficient.

To address these limitations,  we propose \textbf{TagBO} (\textbf{T}ask-\textbf{a}ware \textbf{G}raph \textbf{B}ayesian \textbf{O}ptimization), a method that integrates LLM-derived relational priors with structural information for mixed-variable Bayesian optimization. Rather than directly using LLM outputs as predictive scores, TagBO extracts ordinal relational signals from LLM reasoning to represent task-aware relationships between categorical choices. These signals are combined with structural information derived from physicochemical descriptors, enabling the optimization space to reflect both task-specific functional behavior and local structural similarity. We evaluate TagBO on diverse scientific optimization benchmarks involving chemistry and materials formulation tasks. Experimental results show that TagBO consistently improves sample efficiency under fixed experimental budgets and remains effective across LLMs of varying capabilities. Our contributions are as follows:
\begin{itemize}[
    leftmargin=0.8em,
    itemsep=2pt,
    topsep=2pt,
    parsep=0pt
]
\item \emph{Task-aware similarity modeling for scientific optimization.}
We propose TagBO, a method that incorporates LLM-derived relational knowledge to construct task-aware similarity structures for mixed-variable Bayesian optimization.
\item \emph{Ordinal relational modeling from LLM reasoning.}
TagBO treats LLM outputs as a potentially noisy source of task-aware relational priors, and extracts coarse ordinal signals from LLM reasoning for task-aware similarity modeling in scientific optimization problems.
\item \emph{Improved sample efficiency under fixed budgets.}
Across diverse mixed-variable benchmarks, TagBO consistently improves optimization performance under fixed experimental budgets and remains robust across LLMs of varying capabilities.
\end{itemize}

\section{Related Work}
\paragraph{Structured representations in Bayesian optimization.}
Many approaches construct similarity structures from domain-specific descriptors, learned embeddings, or latent representations \cite{gomez2018automatic,frazier2015materials} to support Bayesian optimization in scientific design problems, such as LADDER, LOLBO~\cite{deshwal2021combining,maus2022local}. 
These methods enable the surrogate model to generalize across large combinatorial or molecular search spaces through representation-induced similarity. However, these approaches are primarily designed for optimizing or generating individual molecular structures in continuous latent spaces, and are less suitable for mixed-variable formulation problems involving constrained combinations of heterogeneous categorical and continuous variables.

In contrast, prior work in discrete and combinatorial Bayesian optimization has focused on defining explicit similarity structures to enable generalization across candidate configurations. For example, COMBO~\cite{oh2019combo} represents the search space as a graph Cartesian product and constructs kernels over this structure, allowing the surrogate model to generalize via graph-induced similarity. CoCaBO~\cite{ru2020bayesian} further addresses mixed continuous and categorical optimization through kernels that share information across categorical variables. Across these approaches, such representations are typically task-agnostic and are not explicitly aligned with the task-specific structure of the objective, which can lead to model misspecification and misleading exploration~\cite{sollich2004robustgp}, limiting their ability to perform targeted and diverse exploration.

\paragraph{LLM-enhanced Bayesian optimization.}
Recent work has explored integrating LLMs into Bayesian optimization to guide candidate generation and improve search efficiency. LLAMBO~\cite{Liu_2024_ICLR} leverages LLMs to propose promising initial points and generate candidate solutions during optimization, while relying on standard acquisition functions for selection. BOPRO~\cite{Agarwal_2025_ICLR} performs optimization in a latent space induced by pretrained LLM embeddings. 
CAKE~\cite{suwandiadaptive} employs LLMs to adapt Gaussian process kernels during optimization, enabling more flexible inductive biases for Bayesian optimization. More recently, LGBO~\cite{yuan2026unleashing} further incorporates LLMs to guide the search process through generation or scoring mechanisms. And LABO~\cite{chen2026labo} incorporates LLM-generated predictions into Bayesian optimization through a confidence-gated dual-surrogate framework.  However, existing approaches do not explicitly construct interpretable task-aware structures over the search space, which is particularly important in scientific discovery problems where understanding relationships among candidate configurations is crucial.

\section{Method}
\begin{figure}[t]
    \centering
    \includegraphics[width=1\textwidth]{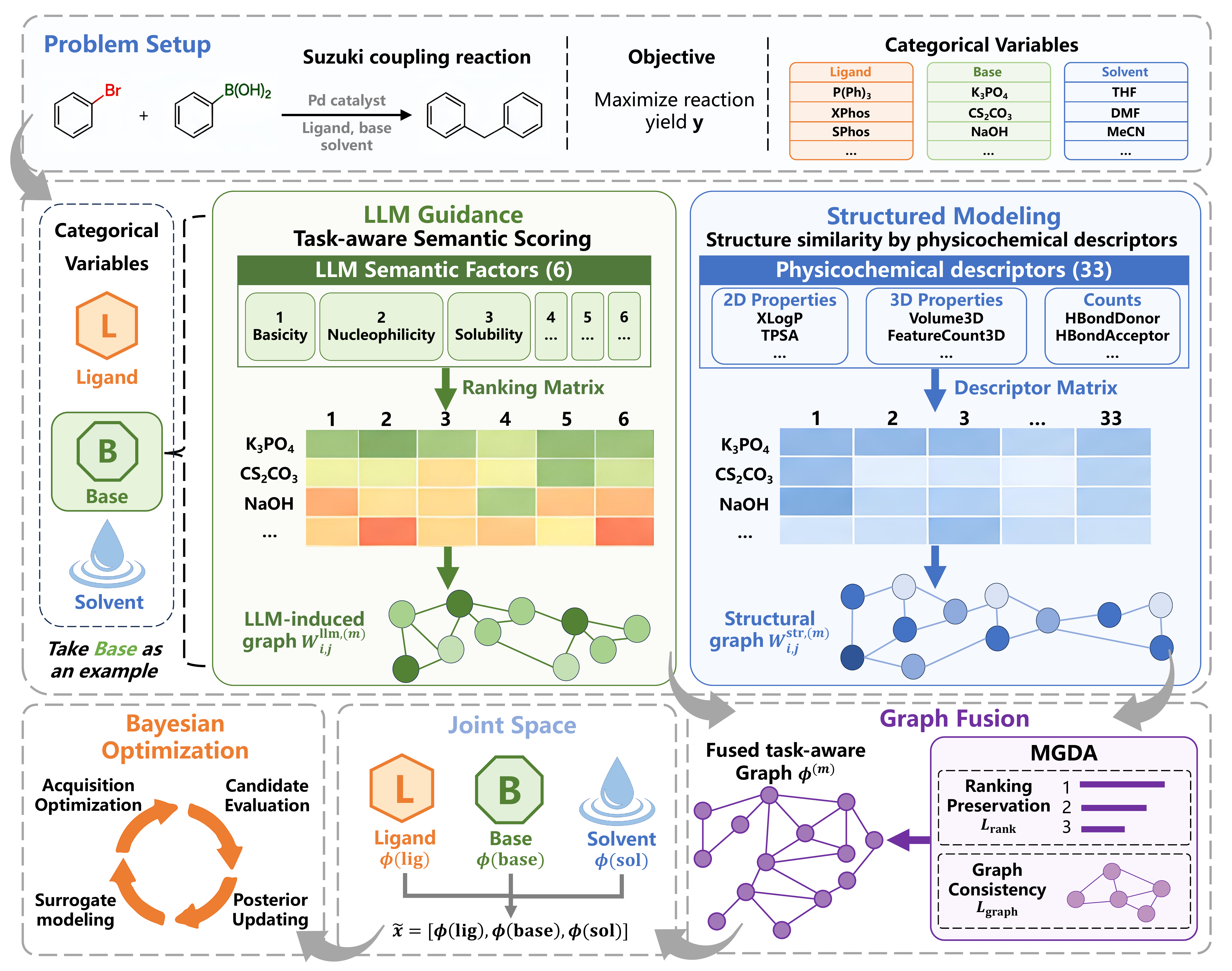}
\caption{
Overview of TagBO on Suzuki coupling reaction optimization. For each categorical variable, TagBO constructs an LLM-derived task-aware relational graph from semantic rankings under the target task, alongside a structural graph derived from physicochemical descriptors. The two signals are fused to form a task-aware optimization space, which is combined with continuous variables and used for Bayesian optimization.
}
    \label{fig:lgbo_method}
\end{figure}
\paragraph{Problem Setup.} We consider black-box optimization problems arising in scientific discovery, where the goal is to identify high-performing experimental configurations from a mixed search space. Formally, we optimize an unknown objective function over
\begin{equation}
\max_{(\mathbf{x}, \mathbf{c}) \in \mathcal{X} \times \mathbb{R}^N}
f(\mathbf{x}, \mathbf{c}),
\end{equation}
where $\mathbf{x}=(x^{(1)},\ldots,x^{(M)}) \in \mathcal{X}=\prod_{m=1}^{M}\mathcal{X}^{(m)}$ denotes a configuration of $M$ discrete decision variables, and $\mathbf{c}\in\mathbb{R}^{N}$ denotes continuous experimental variables. Each $\mathcal{X}^{(m)}$ is a finite set of candidates with explicit semantic meaning (e.g., reagents, solvents, molecular structures, or material components), where each candidate corresponds to a physically interpretable component in the experimental system. The continuous variables correspond to controllable conditions such as concentration, temperature, or reaction time.

The main difficulty lies in defining a task-aware similarity structure over the discrete space $\mathcal{X}$, which limits the applicability of standard Bayesian optimization methods. 
In addition, in scientific optimization, each categorical candidate corresponds to a physically meaningful entity, so the learned representation should preserve a one-to-one correspondence with the original candidate set. 
To this end, TagBO adopts a graph-based representation, where nodes correspond to valid categorical candidates and edges encode task-aware relational similarities. 
The key idea is to construct this graph by combining coarse-grained ordinal relational guidance from LLMs with fine-grained local structural information from external descriptors, thereby inducing a task-aware similarity structure over $\mathcal{X}$. 
Bayesian optimization is then performed over the resulting representation while respecting the feasibility constraints of the original mixed domain.

We organize the method as follows. Sec.~\ref{sec:3.1}--\ref{sec:3.3} construct task-aware representations for different discrete variables. Sec.~\ref{sec:3.4} builds the joint representation and integrates it with Bayesian optimization. Sec.~\ref{sec:3.5} provides representation-Level interpretation.

\subsection{Coarse-grained LLM Guidance}\label{sec:3.1}

Direct numerical scores produced by LLMs can be noisy or inconsistent across candidates and tasks. To obtain a more stable form of task-aware prior information, we use LLMs to extract coarse-grained ordinal relational guidance instead of precise predictive scores.

For each discrete variable type $\mathcal{X}^{(m)}$, we leverage an LLM to extract a set of task-relevant factors and to characterize candidates along these dimensions. Specifically, given the candidate set $\{x_i^{(m)}\}_{i=1}^{N} \subset \mathcal{X}^{(m)}$, the LLM first identifies $F$ key factors that are likely to influence the target objective based on its prior knowledge and the task context. Subsequently, the LLM ranks the $N$ candidates for each factor $f$ according to their relative strengths, resulting in a ranking matrix $R^{(m)} \in \{1,\dots,N\}^{N \times F}$, where $R^{(m)}_{i,f}$ denotes the rank of candidate $x_i^{(m)}$ under factor $f$. We represent each candidate using its ranking vector
\[\mathbf{r}_i^{(m)} = (R^{(m)}_{i,1}, \dots, R^{(m)}_{i,F}) \in \{1,\dots,N\}^F,\] 
which provides a coarse relational representation of candidate $x_i^{(m)}$. We interpret these ranking vectors as elements in a task-aware space $\mathcal{Z}^{(m)} := \mathbb{R}^F$ and normalize them to obtain continuous embeddings $\mathbf{z}_i^{(m)} \in \mathcal{Z}^{(m)}$. We then define pairwise distances $d_{ij}^{(m)} = \|\mathbf{z}_i^{(m)} - \mathbf{z}_j^{(m)}\|_2$, and construct a similarity matrix using an RBF kernel:
\begin{equation}
W^{\text{llm},(m)}_{ij}
=
\exp\left(
-\gamma
\|\mathbf{z}_i^{(m)} - \mathbf{z}_j^{(m)}\|_2^2
\right),
\end{equation}
where $\gamma > 0$ is a bandwidth parameter controlling the decay rate. The resulting matrix $W^{\text{llm},(m)} \in \mathbb{R}^{N \times N}$ defines a weighted graph over categorical candidates, where edge weights encode task-aware relational similarity induced from LLM reasoning. This construction converts coarse-grained ordinal relationships into a task-aware similarity structure, which later serves as relational guidance for graph regularization and Bayesian optimization.
    
    \subsection{Fine-grained Structured Modeling}\label{sec:3.2}
    
While the LLM-induced graph represents coarse-grained ordinal relational information, ranking-based signals provide only relative task-aware structure and may contain noise from LLM reasoning. They therefore do not fully characterize fine-grained local consistency among candidates. To stabilize the resulting representation space, we introduce descriptor-based structural smoothness as a complementary inductive bias. Specifically, we incorporate domain-specific descriptors obtained from external databases (e.g., PubChem~\cite{kim2016pubchem}) for each discrete variable type $\mathcal{X}^{(m)}$. Given the candidate set $\{x_i^{(m)}\}_{i=1}^{N} \subset \mathcal{X}^{(m)}$, we represent each candidate by a structural attribute vector $\mathbf{a}_i^{(m)} \in \mathbb{R}^K$, where $K$ denotes the number of descriptor features. We measure pairwise similarity in this structural space using an RBF kernel:
    \begin{equation}
    W^{\text{str},(m)}_{ij} = \exp\left(-\gamma_{\text{str}} \|\mathbf{a}_i^{(m)} - \mathbf{a}_j^{(m)}\|_2^2 \right),
    \end{equation}
    where $\gamma_{\text{str}} > 0$ is a bandwidth parameter. The resulting matrix $W^{\text{str},(m)} \in \mathbb{R}^{N \times N}$ defines a structural similarity graph over the same candidate set as $W^{\text{llm},(m)}$, and thus shares the same dimensionality. This provides a complementary view of candidate relationships, capturing fine-grained consistency that may not be reflected in the LLM-derived task-aware space.
    
\subsection{Interpretation as Implicit Graph Fusion.}\label{sec:3.3}

To obtain a stable task-aware representation for Bayesian optimization, we combine coarse-grained ordinal relational guidance from the LLM with fine-grained structural smoothness from descriptor-based similarities. For each discrete variable type $\mathcal{X}^{(m)}$, let $\mathbf{z}_i^{(m)} \in \mathcal{Z}^{(m)}$ denote the task-aware representation of candidate $x_i^{(m)}$ defined in Sec.~\ref{sec:3.1}. Our goal is to learn a refined representation that preserves the coarse relational ordering induced by the LLM while remaining locally smooth over structurally similar candidates. Specifically, we consider the following two objectives:

\paragraph{(1) Ranking preservation.}
We define a soft ranking loss that encourages the learned representation to preserve the relative ordering induced by the LLM. For each factor $f$, the ranking matrix $R^{(m)}$ provides pairwise ordering information. For any pair of candidates $(i,j)$ such that $R^{(m)}_{i,f} < R^{(m)}_{j,f}$ (i.e., $x_i^{(m)}$ is ranked higher than $x_j^{(m)}$ under factor $f$), we enforce that their learned representations maintain this ordering. The ranking loss is defined  as:
\begin{equation}
\mathcal{L}_{\text{rank}}^{(m)}(\mathbf{z}) 
= \sum_{f=1}^{F} \sum_{i,j} \mathbb{I}(R^{(m)}_{i,f} < R^{(m)}_{j,f}) \cdot 
\log\left(1 + \exp\left( - (z^{(m)}_{i,f} - z^{(m)}_{j,f}) \right)\right),
\end{equation}
where $\mathbb{I}(\cdot)$ is the indicator function. This loss penalizes violations of pairwise ordering, encouraging the learned representation to remain consistent with the LLM-induced task-aware ranking while allowing flexibility under structural constraints.

\paragraph{(2) Graph consistency.}
We enforce smoothness over both the task-aware graph $W^{\text{llm},(m)}$ and the structural similarity graph $W^{\text{str},(m)}$:
\begin{equation}
\mathcal{L}_{\text{graph}}^{(m)}(\mathbf{z}) 
= \sum_{i,j} \left(\sqrt{W^{\text{llm},(m)}_{ij} \, W^{\text{str},(m)}_{ij}} \, \|\mathbf{z}_i^{(m)} - \mathbf{z}_j^{(m)}\|_2\right)^2.
\end{equation}

The multiplicative form can be interpreted as a geometric-mean style fusion between the task-aware relational graph and the structural similarity graph. As a result, smoothness is primarily enforced between candidate pairs that are simultaneously supported by both the coarse ordinal prior and the local structural similarity. Then, the overall problem is thus formulated as a multi-objective optimization:
\begin{equation}\label{equ:mgda}
\min_{\mathbf{z}^{(m)}} \quad 
\left( 
\mathcal{L}_{\text{rank}}^{(m)}(\mathbf{z}), \;
\mathcal{L}_{\text{graph}}^{(m)}(\mathbf{z})
\right).
\end{equation}

By jointly optimizing these two objectives using the Multi-Gradient Descent Algorithm~\citep{desideri2009multiple}, we obtain a refined task-aware representation that balances coarse ordinal guidance from the LLM with fine-grained structural smoothness from descriptor similarities. Rather than manually specifying a fixed trade-off coefficient, MGDA adaptively balances these two inductive biases during optimization.

\subsection{Joint Optimization Space and Constrained BO}\label{sec:3.4}

We now construct a joint optimization space by combining the refined representations of all variable types. Suppose there are $M$ types of discrete decision variables, where the refined representation of a candidate from the $m$-th type lies in $\mathcal{Z}^{(m)} = \mathbb{R}^{F_m}$. A joint configuration in the representation space is obtained by concatenating the embeddings of selected candidates:
\begin{equation}
\tilde{\mathbf{x}} = [\mathbf{z}^{(1)}, \mathbf{z}^{(2)}, \dots, \mathbf{z}^{(M)}]
\in \mathbb{R}^{\sum_{m=1}^{M} F_m},
\end{equation}

where each $\mathbf{z}^{(m)} \in \mathcal{Z}^{(m)}$ corresponds to the refined representation of a selected candidate $x^{(m)} \in \mathcal{X}^{(m)}$. In addition, continuous decision variables $\mathbf{c} \in \mathbb{R}^N$ are appended to $\tilde{\mathbf{x}}$, yielding a unified representation space:
\begin{equation}
\tilde{\mathcal{X}} = \left\{ 
[\mathbf{z}^{(1)}, \dots, \mathbf{z}^{(M)}, \mathbf{c}] \;\middle|\;
x^{(m)} \in \mathcal{X}^{(m)}, \; \mathbf{z}^{(m)} = \phi^{(m)}(x^{(m)}), \; \mathbf{c} \in \mathcal{C}
\right\},
\end{equation}
where $\phi^{(m)}$ denotes the representation mapping defined in Sec.~\ref{sec:3.1}--\ref{sec:3.3}, and $\mathcal{C}$ denotes the domain of continuous variables. We perform Gaussian process Bayesian optimization in this representation space, with acquisition optimization restricted to the feasible set:
\begin{equation}
\tilde{\mathbf{x}}_{t+1}
= \arg\max_{\tilde{\mathbf{x}} \in \tilde{\mathcal{X}}} \alpha_t(\tilde{\mathbf{x}}),
\end{equation}
ensuring that only valid configurations are considered during the search.

\begin{remark}[Search Complexity]
Restricting optimization to $\tilde{\mathbf{x}}$ ensures that only valid configurations are evaluated. As $\tilde{\mathbf{x}}$ coincides with the original discrete candidate set, both the search space cardinality and the number of acquisition evaluations remain unchanged. Thus, the proposed representation introduces no additional computational overhead.
\end{remark}

\subsection{Representation-Level Interpretation}\label{sec:3.5}

We provide several representation-level interpretations to better understand why task-aware representations may improve Bayesian optimization performance in mixed-variable scientific search.
Our discussion focuses on three aspects:
(i) Gaussian process modeling under representation warping,
(ii) the relationship between representation quality and surrogate approximation behavior,
and (iii) the adaptive trade-off between task-aware relational guidance and structural smoothness.
Due to space limitations, technical assumptions and derivations are deferred to Appendix~\ref{proof}.

\paragraph{LLM relational priors as surrogate priors.}

Our method can be interpreted as transforming ordinal relational knowledge derived from LLM reasoning into a Gaussian process prior over the search space.
Rather than directly using LLM outputs as predictive scores, TagBO incorporates relative ranking relations into the learned representation space prior to surrogate modeling. Let $w:\mathcal{X}\to\mathcal{Z}$ denote the learned representation map induced by the fused relational-structural graph, and let
$g\sim\mathcal{GP}(m,K)$ be a Gaussian process defined on the representation space $\mathcal{Z}$.
Defining $f(x)=g(w(x))$, induces a valid Gaussian process on the original input space with kernel
\[
K_w(x,x')
=
K(w(x),w(x')) .
\]

Under this interpretation, LLM-derived ordinal relations are not used as direct predictions, but instead shape the similarity structure and inductive bias of the surrogate model itself.
This connects task-aware relational guidance with probabilistic surrogate modeling in Bayesian optimization.

\paragraph{LLM ordinal priors and surrogate approximation.}
Bayesian optimization critically depends on whether the surrogate model can generalize smoothly across the search space.
However, structural similarity alone may fail to reflect task-relevant behavior in scientific optimization problems. This motivates using LLM-derived ordinal relations as an additional task-aware source of information.
Here, we analyze the representation-level effect of the LLM-induced representation relative to purely structural and one-hot representations; the subsequent graph-smoothing formulation further regularizes this signal with structural similarity.
To formalize this intuition, let
$\mathcal{D}_W
=
\mathbb{E}
\left[
\left\|
Z^\star-\mathbb{E}[Z^\star\mid W]
\right\|^2
\right]$
denote the latent reconstruction error associated with a representation $W$, where $Z^\star$ denotes a latent task-relevant variable.
Under the representation-quality assumptions described in Appendix~\ref{proof}, if the LLM-induced ordinal representation preserves task-relevant variation more effectively than task-agnostic structural descriptors, then
\[
\mathcal{D}_L
\le
\mathcal{D}_S
\le
\mathcal{D}_O .
\]

Furthermore, if the task link $f=h(Z^\star)$ is $L_h$-Lipschitz, the corresponding surrogate approximation upper bounds satisfy
\[
\mathcal{U}_L
\le
\mathcal{U}_S
\le
\mathcal{U}_O .
\]

This interpretation does not claim that LLM information is always superior; rather, it formalizes the condition under which LLM-derived ordinal priors can provide a more task-aware surrogate prior than purely structural similarity. The role of the structural graph in TagBO is complementary: it regularizes and smooths this LLM-induced relational signal during representation learning.

\paragraph{Adaptive trade-off interpretation.}
Our multi-objective formulation jointly optimizes
ordinal relational consistency and structural graph smoothness in (\ref{equ:mgda}). Under standard MGDA~\cite{desideri2009multiple} optimality conditions (Appendix~\ref{proof}),
a Pareto-stationary solution $\mathbf{z}^\star$ admits an equivalent adaptive scalarization:
\[
\nabla
\left(
\mathcal{L}_{\mathrm{rank}}(\mathbf{z}^\star)
+
\lambda^\star
\mathcal{L}_{\mathrm{graph}}(\mathbf{z}^\star)
\right)
=0 ,
\]
for some adaptive coefficient $\lambda^\star >0$. The ranking objective preserves coarse-grained ordinal relations induced by LLM reasoning, while the structural objective regularizes these relations through local similarity constraints.
This interpretation suggests that the learned representation emerges from balancing coarse-grained ordinal relational guidance with fine-grained local structural smoothness.

\section{Experiments}

We evaluate the performance of TagBO on multiple optimization datasets from the domains of pharmaceuticals, chemistry, and materials science. These tasks are designed to provide a unified comparison of optimization efficiency and stability under limited evaluation budgets, benchmarking against both traditional Bayesian optimization methods and recent LLM-guided approaches. More experiments for stability and ablation are also provided in Appendix~\ref{app:exp}.

\subsection{Experimental Setup}

\paragraph{Baselines.} The following methods are used as baselines:

\begin{enumerate}[label=\arabic*), leftmargin=*]

\item \textbf{One-hot Feature BO:} Discrete variables are encoded using one-hot representations, and the concatenated sparse vectors are fed into a Bayesian optimization model for iterative search.

\item \textbf{Property Feature BO:} Physicochemical properties of candidate substances are extracted and concatenated into continuous feature vectors, which are then fed into a Bayesian optimization model for iterative search.

\item \textbf{COMBO~\cite{oh2019combo}:} A graph structure is constructed over the discrete combinatorial space, and Gaussian process modeling and optimization are performed based on graph kernels.

\item \textbf{BOPRO~\cite{Agarwal_2025_ICLR}:} Iteratively updates the LLM prompt using selected past generations to sample new candidates that explore or exploit different regions of the search space.

\item \textbf{LGBO~\cite{yuan2026unleashing}:} This method incorporates LLM predictions as low-cost evaluation signals into the BO method alongside real experimental results. A gating mechanism is used to dynamically balance the use of LLM predictions and real evaluations, thereby improving optimization efficiency.
\end{enumerate}

\paragraph{Tasks.} We evaluate all methods on six optimization tasks spanning chemistry, materials science, and energy-related applications. Detailed data descriptions are provided in \ref{app:data}:

\begin{enumerate}[label=\arabic*), leftmargin=*]

\item \textbf{Organic Photovoltaics \citep{mmli_opv} (2200 combinations):} Maximize spectral overlap predicted by DFT by optimizing donor--acceptor building blocks in organic photovoltaic materials.

\item \textbf{Suzuki--Miyaura Reaction \citep{Suzuki} (384 combinations):} Maximize reaction yield on a fixed-substrate Suzuki--Miyaura reaction subset by optimizing three categorical reaction conditions: ligand, base, and solvent.

\item \textbf{Organic Laser Molecules \citep{dye_lasers} (3458 combinations):} Optimize the photophysical performance of symmetric pentamer molecules by selecting three molecular fragments, with the goal of maximizing peak score and fluorescence rate while minimizing spectral overlap.

\item \textbf{Perovskite Bandgap \citep{perovskites} (192 combinations):} Predict and tune the band gap of hybrid organic--inorganic perovskites by selecting the organic, cation, and anion components.

\item \textbf{Redox Flow Battery \citep{redoxmer} (1408 combinations):} Optimize redoxmer candidates for non-aqueous redox flow batteries by tuning four molecular substituents to minimize absorption wavelength difference, reduction potential, and solvation free energy.

\item \textbf{Crystallization Process \citep{vapdiff_crystal} (mixed-variable):} Optimize the crystallization outcome in a vapor diffusion experiment by adjusting solvent identity, concentrations, vial volumes, reaction time, and temperature to maximize the ordinal crystal score.
\end{enumerate}

\paragraph{Implementation details.} All methods employ Gaussian process (GP)-based surrogate models with an RBF kernel, and optimization for each task is initialized with 10 starting points. In the main experiments, the LLM used is Intern S1 241B Thinking, a large language model pretrained on scientific knowledge. All experiments were conducted using 4 NVIDIA H200 GPUs.

\subsection{Main Results}
% As shown in Fig. \ref{fig:main_1}, TagBO demonstrates faster performance improvement across all six optimization tasks, reaching high-quality regions earlier than the baselines. Its advantage is most evident in the early and middle stages, where it identifies strong candidates more efficiently under the same evaluation budget. Despite this, TagBO does not sacrifice final solution quality. It achieves comparable or better final performance while requiring fewer evaluations to reach high-quality solutions. This suggests that leveraging LLM-derived prior knowledge helps guide the discrete search more effectively and reduces unnecessary exploration. Such efficiency is particularly beneficial for materials and molecular optimization tasks with high experimental costs.

\begin{figure}[t]
    \centering
    \includegraphics[width=1\textwidth]{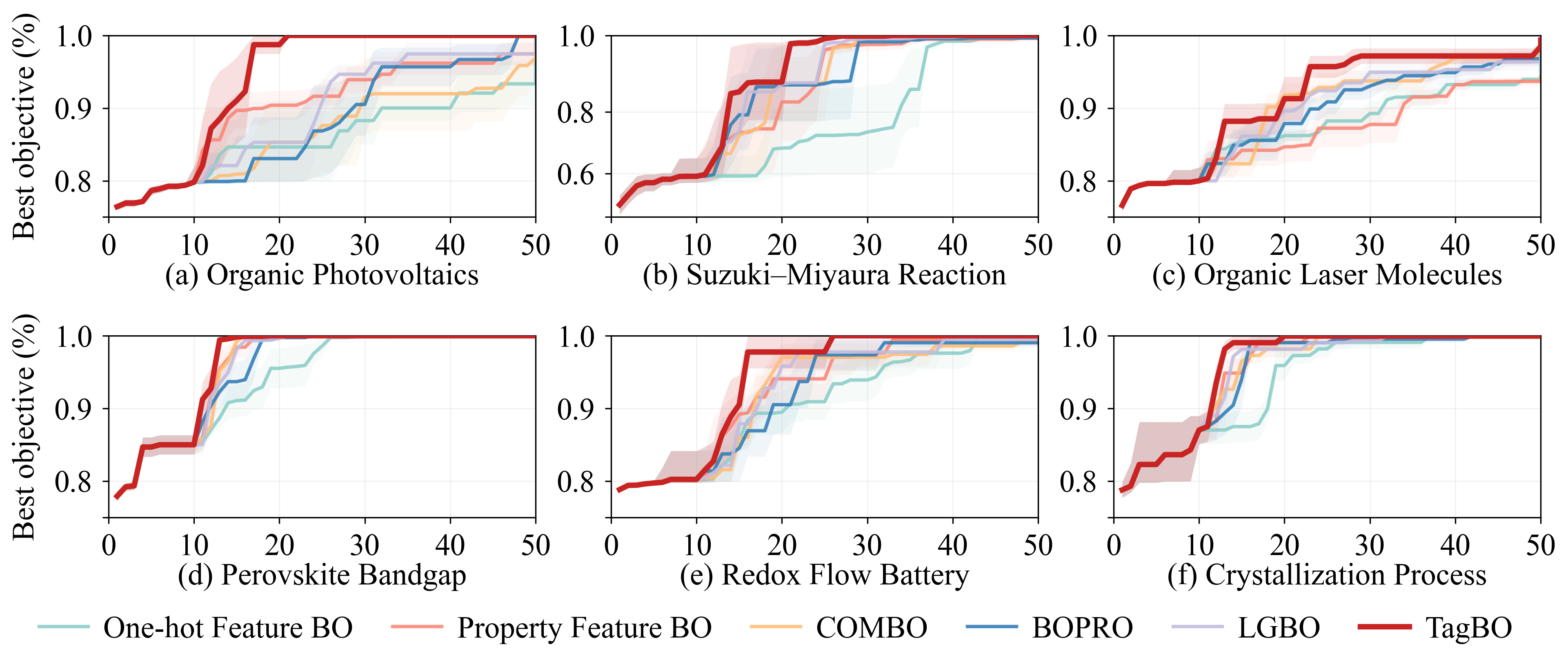}
    \caption{Optimization trajectories on six tasks. The y-axis reports the best objective value found so far, normalized by the oracle optimum, and the x-axis reports the iteration number. TagBO is compared against One-hot Feature BO, Property Feature BO, COMBO, BOPRO, and LGBO.}
    \label{fig:main_1}
\end{figure}

\begin{figure}[t]
    \centering
    \includegraphics[width=1.0\textwidth]{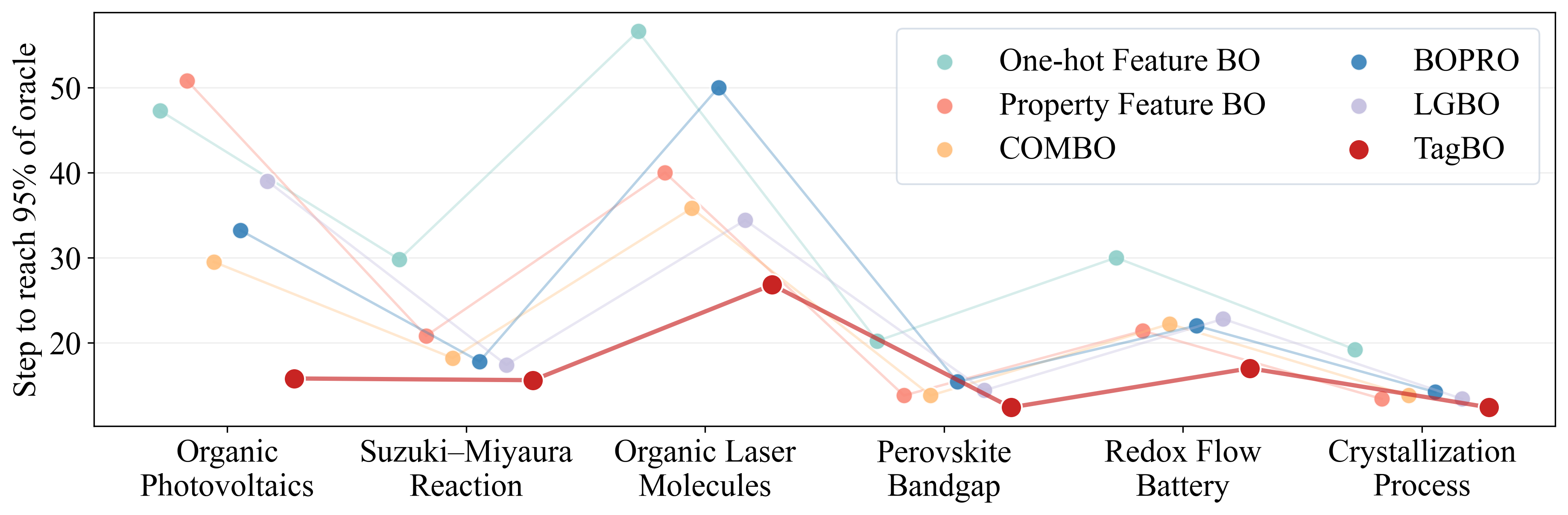}
    \caption{Comparison of the average number of iterations required to reach 95\% of the oracle optimum on six benchmark tasks. Lower is better.}
    \label{fig:main_2}
\end{figure}

As shown in Fig.~\ref{fig:main_1} and Fig.~\ref{fig:main_2}, TagBO consistently improves sample efficiency across all six optimization tasks, particularly during the early and middle stages of optimization. Under the same evaluation budget, TagBO identifies high-quality candidates with fewer evaluations while maintaining competitive final performance. This advantage is reflected quantitatively by lower Step@95\% values across all datasets, indicating faster convergence toward near-optimal solutions. The improvement is especially notable on challenging discrete optimization tasks such as Organic Photovoltaics and Organic Laser Molecules  These results suggest that the proposed representation effectively incorporates task-relevant relational priors into the BO search space, improving early-stage search efficiency in expensive scientific optimization settings.

In addition, to understand why the learned representation improves optimization performance, we visualize the reaction space induced by different feature sets using a 3D UMAP embedding~\cite{mcinnes2018umap}, as shown in Fig.~\ref{fig:feature}. Compared with one-hot and property-feature baselines, the representation learned by TagBO induces a more connected and smoother high-performing region in the projected search space. High-yield candidates become more locally clustered and exhibit improved neighborhood continuity, suggesting that the learned representation forms a more optimization-friendly landscape for Gaussian-process-based surrogate modeling. Such structure is beneficial for acquisition-guided search, as local interpolation and uncertainty estimation become more aligned with task performance.
\vspace{-1.5em}
\subsection{Ablation Experiences}

\begin{figure}[t]
    \centering
    \includegraphics[width=1\textwidth]{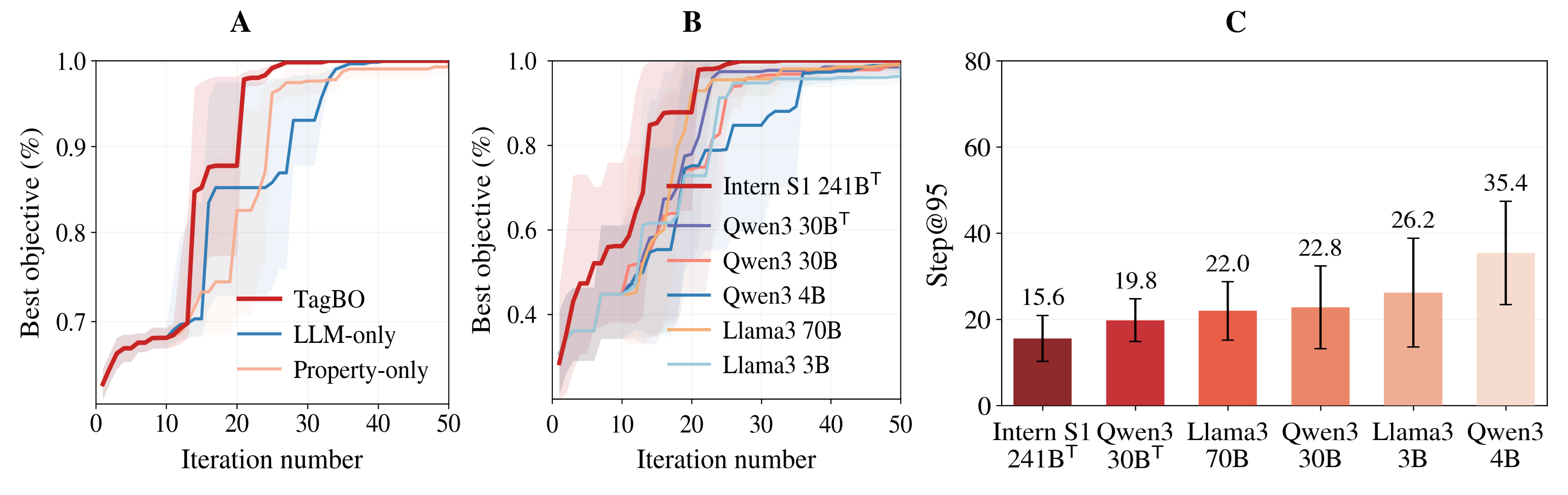}
    \caption{Ablation analysis on the Suzuki dataset. 
(A) Comparison with two reduced variants using only LLM-derived scoring features or only physicochemical property features. 
(B) Optimization trajectories by different LLMs, where curves report the best objective value found so far, normalized by the oracle optimum. 
(C) Mean Step@95\% under different LLMs. $*^T$ denote Thinking variants.}
    \label{fig:ablation_1}
\end{figure}

\begin{figure}[t]
    \centering
    \includegraphics[width=1\textwidth]{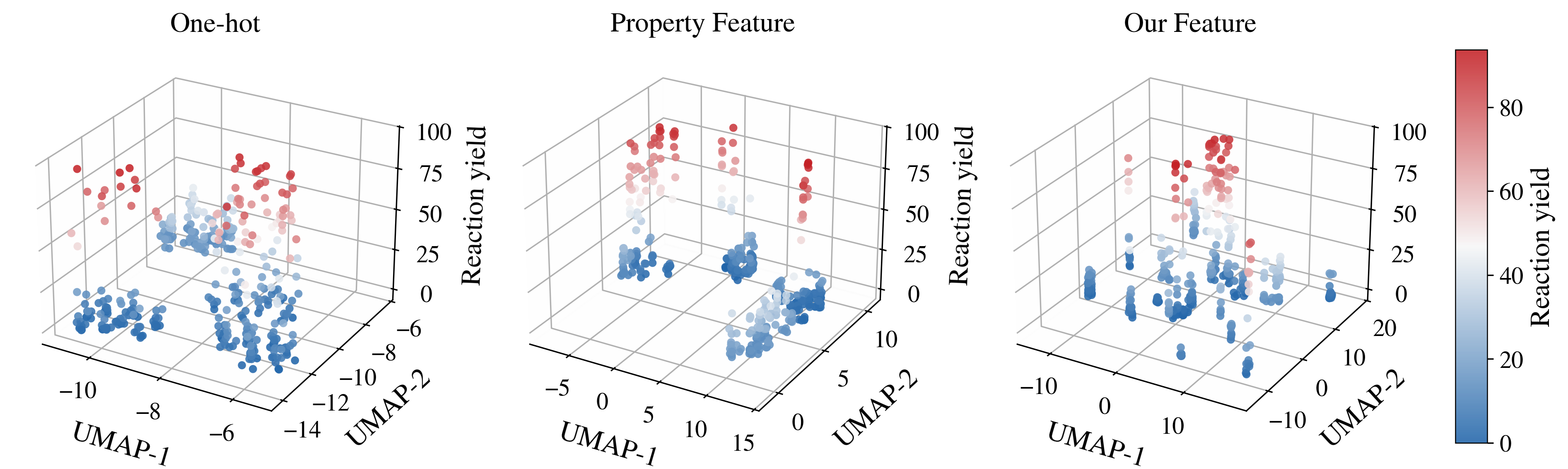}
    \caption{3D UMAP visualization of feature spaces. Each point represents a reaction candidate; the horizontal axes show the UMAP embedding, and the vertical axis encodes the reaction yield.}
    \label{fig:feature}
\end{figure}
\begin{figure}[t]
    \centering
    \includegraphics[width=0.7\linewidth]{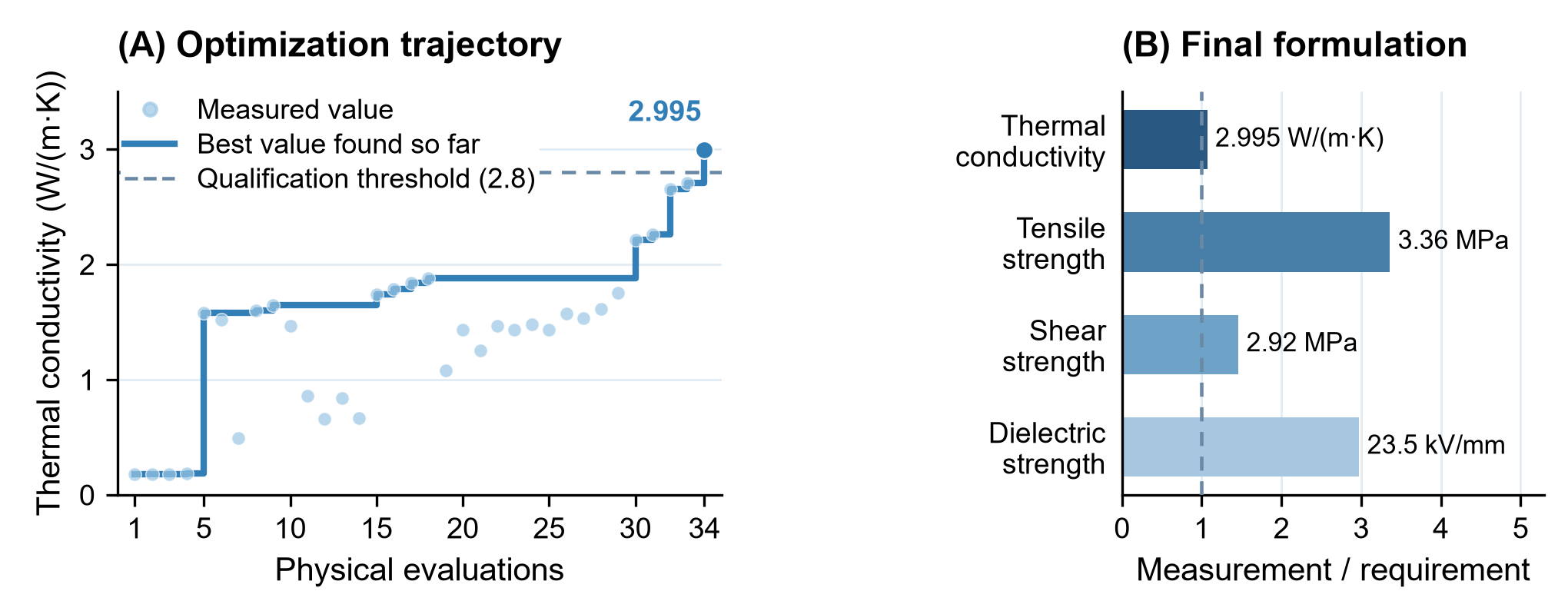}
    \caption{Real-world formulation optimization. \textbf{(A)} Thermal conductivity and the best value found so far over 34 physical evaluations. \textbf{(B)} Final measurements relative to their qualification thresholds; values above one satisfy the requirements.}
    \label{fig:real_exp}
\end{figure}

\vspace{-0.5em}

\paragraph{Representation source ablation.}
To investigate the roles of relational guidance and structural smoothness, we compare the full TagBO model with two reduced variants that use only LLM-derived relational features or only physicochemical property features. As shown in Fig.~\ref{fig:ablation_1}(A), the LLM-only variant achieves strong early-stage optimization performance, indicating that ordinal relational priors provide effective task-aware guidance during early search. However, its performance gradually saturates in later stages, suggesting that relational priors alone are insufficient for fine-grained surrogate refinement. In contrast, the structural-only variant exhibits consistently weaker performance throughout the optimization process. By combining coarse-grained relational guidance with fine-grained structural smoothness, TagBO achieves both strong early-stage exploration and improved later-stage refinement.

\paragraph{Effect of LLM prior quality.}
We compare LLM backbones with varying scales and reasoning capabilities, including \textbf{Intern-S1-241B}~\cite{bai2025intern}, \textbf{Llama-3.3-70B}~\cite{touvron2023llama}, \textbf{Qwen3-30B}~\cite{bai2023qwen}(Instruct and Thinking), as well as smaller variants such as \textbf{Qwen3-4B} and \textbf{Llama-3.2-3B}. As shown in Fig.~\ref{fig:ablation_1}(B) and Fig.~\ref{fig:ablation_1}(C), stronger models generally achieve lower Step@95\% values and improved final sample efficiency. In particular, \textbf{Intern-S1-241B Thinking} consistently reaches 95\% of the oracle optimum using the fewest evaluations, suggesting that domain-relevant scientific pretraining provides more informative relational priors for guiding optimization. Interestingly, optimization trajectories from different LLM backbones remain relatively close during early search, suggesting that the proposed ordinal relational formulation is less sensitive to precise numerical calibration from a specific LLM. Stronger models mainly improve later-stage refinement by providing more informative task-aware relational structure.
\subsection{Real-world Formulation Experiment}\label{app:real_exp}

\paragraph{Task setting.}
This task aims to maximize the measured \textbf{thermal conductivity} of a thermally conductive adhesive while satisfying mechanical and electrical performance requirements. The mixed-variable search space involves selecting 10 fillers from a library of 29 candidate materials and optimizing their 10 corresponding continuous composition variables, as summarized in Table~\ref{tab:real_exp_variables}. The thermal conductivity qualification threshold is $2.8\,\mathrm{W/(m\cdot K)}$. The formulation must also satisfy tensile strength $\geq 1\,\mathrm{MPa}$, shear strength $\geq 2\,\mathrm{MPa}$, and dielectric strength $\geq 7.9\,\mathrm{kV/mm}$. The experiment includes 34 physical evaluations, each requiring approximately three days.
\paragraph{Results.}
As shown in Fig.~\ref{fig:real_exp}, the best measured thermal conductivity increases from $0.1823\,\mathrm{W/(m\cdot K)}$ at the first evaluation to $2.995\,\mathrm{W/(m\cdot K)}$ after only 34 evaluations, exceeding the qualification threshold. The final formulation achieves a tensile strength of $3.36\,\mathrm{MPa}$, a shear strength of $2.92\,\mathrm{MPa}$, and a dielectric strength of $23.5\,\mathrm{kV/mm}$, satisfying all three performance requirements. Under the $n_{\mathrm{init}}=2d$ setting used in standard TuRBO~\citep{eriksson2019scalable} implementations, a 20-variable problem would require 40 evaluations for initialization alone. In this large combinatorial and continuous search space, TagBO thus meets all four engineering requirements within a total experimental budget smaller than this initial design, demonstrating its sample efficiency for expensive formulation optimization.

 \section{Conclusion}
 In this work, we proposed TagBO, a task-aware graph Bayesian optimization framework for mixed-variable scientific optimization. Unlike conventional approaches that rely primarily on task-agnostic structural similarity, TagBO incorporates coarse-grained ordinal relational priors extracted from LLM reasoning and combines them with fine-grained structural information through adaptive multi-objective optimization. 
{\small
\bibliographystyle{plain}
\bibliography{egbib}
}
%%%%%%%%%%%%%%%%%%%%%%%%%%%%%%%%%%%%%%%%%%%%%%%%%%%%%%%%%%%%

\appendix

\section*{Appendix}
\section{Notation and Variables}

\begin{center}
\renewcommand{\arraystretch}{1.2}
\begin{tabular}{p{0.22\textwidth} p{0.72\textwidth}}
\hline
\textbf{Symbol} & \textbf{Description} \\
\hline

\multicolumn{2}{l}{\textit{Spaces and mappings}} \\
\\
$\mathcal X$ & Input space \\
$\mathcal Z$ & Latent / transformed space \\
$w:\mathcal X \to \mathcal Z$ & Warping (representation mapping) \\
$x \in \mathcal X$ & Input variable \\
$z \in \mathcal Z$ & Transformed input \\
\\
\multicolumn{2}{l}{\textit{Gaussian process construction}} \\
$g \sim \mathcal{GP}(m, K)$ & Gaussian process defined on $\mathcal Z$ \\
$m(\cdot)$ & Mean function \\
$K(\cdot,\cdot)$ & Covariance kernel \\
$f(x) = g(w(x))$ & Induced process on $\mathcal X$ \\
$K_w(x,x')$ & Induced kernel $K(w(x), w(x'))$ \\
\\
\multicolumn{2}{l}{\textit{Latent-variable formulation}} \\
$Z^\star \in \mathbb R^d$ & Task-relevant latent variable \\
$h:\mathbb R^d \to \mathbb R$ & Task link function ($f = h(Z^\star)$) \\
$L_h$ & Lipschitz constant of $h$ \\
$X$ & Random input variable \\
$f$ & Target output variable \\
\\
\multicolumn{2}{l}{\textit{Representations}} \\
$L = w_{\mathrm{llm}}(X)$ & LLM-based representation \\
$S = w_{\mathrm{str}}(X)$ & Structural representation \\
$O = w_{\mathrm{oh}}(X)$ & One-hot representation \\
\\
\multicolumn{2}{l}{\textit{Error measures}} \\
$\mathcal D_W$ & Latent reconstruction error \\
$\mathcal R_W$ & Bayes risk \\
$\mathcal U_W$ & Surrogate upper bound ($L_h^2 \mathcal D_W$) \\
\\
\multicolumn{2}{l}{\textit{Assumption parameters and mappings}} \\
$\sigma_L^2$ & Upper bound on LLM decoding error \\
$\rho_S^2$ & Irreducible uncertainty under structural representation \\
$\psi_L$ & Decoder from $L$ to $Z^\star$ \\
$T$ & Mapping from $S$ to $O$ \\

\hline
\end{tabular}
\newpage
\vspace{0.5em}
\textbf{Table A.1:} Summary of notation used throughout the appendix.
\end{center}
\section{Proof}\label{proof}
\subsection{Warping induces a Gaussian process prior}
\begin{proposition}[Warping induces a Gaussian process prior]
Let $\mathcal X$ be the input space and $\mathcal Z$ be a measurable space. 
Let $w:\mathcal X \to \mathcal Z$ be a measurable mapping. 
Suppose
\[
g \sim \mathcal{GP}(m, K)
\]
is a Gaussian process defined on $\mathcal Z$, with mean function $m:\mathcal Z \to \mathbb R$ and covariance kernel $K:\mathcal Z \times \mathcal Z \to \mathbb R$.

Define a stochastic process $f:\mathcal X \to \mathbb R$ by
\[
f(x) := g(w(x)).
\]

Then $f$ is a Gaussian process on $\mathcal X$, and
\[
f \sim \mathcal{GP}(m \circ w,\; K_w),
\]
where
\[
(m \circ w)(x) = m(w(x)), 
\qquad
K_w(x,x') = K\big(w(x), w(x')\big).
\]

In particular, $K_w$ is a valid positive definite kernel on $\mathcal X$, and hence a fixed warping $w$ induces a Gaussian process prior over functions on $\mathcal X$.
\end{proposition}

\begin{proof}
To show that $f$ is a Gaussian process, it suffices to verify that for any finite set of inputs $x_1,\dots,x_n \in \mathcal X$, the random vector
\[
(f(x_1),\dots,f(x_n))
\]
is jointly Gaussian.

Let $z_i := w(x_i) \in \mathcal Z$ for $i=1,\dots,n$. Then by definition,
\[
(f(x_1),\dots,f(x_n)) = (g(z_1),\dots,g(z_n)).
\]
Since $g \sim \mathcal{GP}(m,K)$, the vector $(g(z_1),\dots,g(z_n))$ is multivariate Gaussian with mean
\[
\big(m(z_1),\dots,m(z_n)\big)
=
\big(m(w(x_1)),\dots,m(w(x_n))\big),
\]
and covariance matrix
\[
\big(K(z_i,z_j)\big)_{i,j=1}^n
=
\big(K(w(x_i),w(x_j))\big)_{i,j=1}^n.
\]

Therefore $(f(x_1),\dots,f(x_n))$ is Gaussian for any finite collection of inputs, and hence $f$ is a Gaussian process. Its mean and covariance functions follow directly:
\[
\mathbb E[f(x)] = m(w(x)), 
\qquad
\mathrm{Cov}(f(x),f(x')) = K(w(x),w(x')).
\]
This proves the result.
\end{proof}
\begin{remark}[Positive definiteness and modeling space]
The construction above preserves positive definiteness of the kernel. 
Specifically, if $K$ is a positive definite kernel on $\mathcal Z$, then the induced kernel
\[
K_w(x,x') := K\big(w(x), w(x')\big)
\]
is also positive definite on $\mathcal X$. This follows directly from the definition of positive definiteness:
for any $x_1,\dots,x_n \in \mathcal X$ and $c_1,\dots,c_n \in \mathbb R$,
\[
\sum_{i,j=1}^n c_i c_j K_w(x_i,x_j)
=
\sum_{i,j=1}^n c_i c_j K\big(w(x_i), w(x_j)\big)
\ge 0.
\]

Consequently, the Gaussian process prior induced by $w$ is well-defined.

From a modeling perspective, this construction can be interpreted as performing Bayesian inference in the transformed space $\mathcal Z$, where a standard kernel $K$ is applied, while the warping $w$ maps the original inputs into this space. 
\end{remark}

\subsection{Ordered Upper Bounds via Task-Relevant Signal Dominance}

\subsubsection{Setup}

Let $(\Omega,\mathcal F,\mathbb P)$ be a probability space.  
Let $X \in \mathcal X$ denote the input, and suppose there exists a latent task-relevant variable
\[
Z^\star \in \mathbb R^d
\]
such that
\[
f = h(Z^\star),
\]
where $h:\mathbb R^d \to \mathbb R$ is measurable.

We consider three representations:
\[
L := w_{\mathrm{llm}}(X), \qquad
S := w_{\mathrm{str}}(X), \qquad
O := w_{\mathrm{oh}}(X).
\]

\subsubsection{Definitions}

\begin{definition}[Latent reconstruction error]
For any representation $W$, define
\[
\mathcal D_W
:=
\mathbb E\left[\left\|Z^\star-\mathbb E[Z^\star\mid W]\right\|^2\right].
\]
\end{definition}

\begin{definition}[Bayes risk]
For any representation $W$, define
\[
\mathcal R_W := \mathbb E\left[(f-\mathbb E[f\mid W])^2\right].
\]
\end{definition}

\subsubsection{Assumptions}
\label{app:ass}
\begin{assumption}[Lipschitz task link]
\label{ass:lipschitz}
The function $h$ is $L_h$-Lipschitz, i.e.,
\[
|h(z_1)-h(z_2)| \le L_h \|z_1-z_2\|,
\qquad \forall z_1,z_2 \in \mathbb R^d.
\]
\end{assumption}

Our ordering result is based on an explicit modeling hypothesis about the relative task-relevance of the three representations. Recall that $Z^\star$ denotes a latent task-relevant variable such that $f=h(Z^\star)$. We interpret the LLM representation $L$ as encoding task-aware information aligned with the downstream task, the structural representation $S$ as encoding task-agnostic structural regularities, and the one-hot representation $O$ as encoding only input identity. The following assumptions formalize this interpretation.

\begin{assumption}[LLM representation preserves task-relevant information]
\label{ass:llm}
Because the LLM-derived representation is assumed to encode task-relevant  information, we assume that $Z^\star$ can be decoded from $L$ with bounded mean-squared error. That is, there exists a measurable function $\psi_L$ such that
\[
\mathbb E\left[\|Z^\star-\psi_L(L)\|^2\right] \le \sigma_L^2.
\]
Here, $\sigma_L^2$ measures the residual task-irrelevant noise in the LLM representation.
\end{assumption}

\begin{assumption}[Structural representation is task-agnostic]
\label{ass:struct}
The structural representation $S$ is assumed to capture generic relational or topological patterns of the input domain, rather than task-specific predictive content. Accordingly, even after observing $S$, a nontrivial amount of uncertainty about the task-relevant latent variable $Z^\star$ remains. We formalize this by assuming that there exists $\rho_S^2>0$ such that
\[
\mathcal D_S
=
\mathbb E\left[\left\|Z^\star-\mathbb E[Z^\star\mid S]\right\|^2\right]
\ge \rho_S^2.
\]
Thus, $\rho_S^2$ represents an irreducible latent reconstruction error associated with the task-agnostic structural representation.
\end{assumption}

\begin{assumption}[Structural representation contains one-hot information]
\label{ass:o-from-s}
The one-hot representation $O$ only records input identity, whereas the structural representation $S$ is assumed to retain at least this identity information together with additional structural information. Therefore, we assume that there exists a measurable function $T$ such that
\[
O=T(S).
\]
Equivalently, the information carried by $O$ is contained in that carried by $S$.
\end{assumption}

\begin{assumption}[Comparative representation assumption]
\label{ass:separation}
We assume that the task-relevant information contained in the LLM representation is stronger than that available from the task-agnostic structural representation. Quantitatively, this means that the LLM decoding error level is no larger than the irreducible latent uncertainty under $S$:
\[
\sigma_L^2 \le \rho_S^2.
\]
This assumption expresses the intended inductive bias that task-aware features are more informative about $Z^\star$ than purely structural features.
\end{assumption}

\subsubsection{Auxiliary Results}

\begin{lemma}[Bayes risk as conditional variance]
\label{lem:bayes-risk}
For any representation $W$,
\[
\mathcal R_W = \mathbb E\left[\operatorname{Var}(f\mid W)\right].
\]
\end{lemma}

\begin{proof}
Since $\mathbb E[f\mid W]$ is the $L^2$ projection of $f$ onto the $\sigma$-field generated by $W$, we have
\[
f = \mathbb E[f\mid W] + \bigl(f-\mathbb E[f\mid W]\bigr),
\]
where the two terms are orthogonal in $L^2$. Taking expectations of squares yields
\[
\mathcal R_W
=
\mathbb E\left[(f-\mathbb E[f\mid W])^2\right]
=
\mathbb E\left[\operatorname{Var}(f\mid W)\right].
\]
\end{proof}

\begin{lemma}[Transfer from latent reconstruction to target prediction]
\label{lem:transfer}
Under Assumption~\ref{ass:lipschitz}, for any representation $W$,
\[
\mathcal R_W \le L_h^2 \mathcal D_W.
\]
\end{lemma}

\begin{proof}
Let $\delta_W := \mathbb E[f\mid W]$. By optimality of conditional expectation,
\[
\mathcal R_W
=
\mathbb E\left[(f-\delta_W)^2\right]
\le
\mathbb E\left[(f-a(W))^2\right]
\]
for any $W$-measurable function $a(W)$.

Choose
\[
a(W) := h\bigl(\mathbb E[Z^\star\mid W]\bigr).
\]
Then
\[
\mathcal R_W
\le
\mathbb E\left[\bigl(h(Z^\star)-h(\mathbb E[Z^\star\mid W])\bigr)^2\right].
\]
By the Lipschitz property of $h$,
\[
\bigl(h(Z^\star)-h(\mathbb E[Z^\star\mid W])\bigr)^2
\le
L_h^2 \left\|Z^\star-\mathbb E[Z^\star\mid W]\right\|^2.
\]
Taking expectations yields the result.
\end{proof}

\begin{lemma}[LLM upper bound]
\label{lem:llm-upper}
Under Assumption~\ref{ass:llm},
\[
\mathcal D_L \le \sigma_L^2.
\]
\end{lemma}

\begin{proof}
By definition,
\[
\mathcal D_L
=
\inf_{\phi}
\mathbb E\left[\|Z^\star-\phi(L)\|^2\right].
\]
Taking $\phi = \psi_L$ yields
\[
\mathcal D_L
\le
\mathbb E\left[\|Z^\star-\psi_L(L)\|^2\right]
\le
\sigma_L^2.
\]
\end{proof}

\begin{lemma}[Structural representation dominates one-hot]
\label{lem:s-better-o}
Under Assumption~\ref{ass:o-from-s},
\[
\mathcal D_S \le \mathcal D_O.
\]
\end{lemma}

\begin{proof}
Since $O = T(S)$, any $O$-measurable function is also $S$-measurable. In particular, $\mathbb E[Z^\star\mid O]$ is $S$-measurable.

By optimality of conditional expectation,
\[
\mathbb E\left[\|Z^\star-\mathbb E[Z^\star\mid S]\|^2\right]
\le
\mathbb E\left[\|Z^\star-a(S)\|^2\right]
\]
for any $S$-measurable function $a(S)$. Choosing $a(S)=\mathbb E[Z^\star\mid O]$ yields
\[
\mathcal D_S \le \mathcal D_O.
\]
\end{proof}

\subsubsection{Main Result}

\begin{proposition}[Ordered surrogate upper bounds]
\label{thm:order}
Let
\[
\mathcal D_W
=
\mathbb E\left[\|Z^\star - \mathbb E[Z^\star \mid W]\|^2\right]
\]
denote the latent reconstruction error of a representation $W$, i.e., the minimal mean-squared error for recovering the task-relevant latent variable $Z^\star$ from $W$.

Let
\[
\mathcal U_W := L_h^2 \mathcal D_W
\]
denote the latent-reconstruction-induced surrogate upper bound on the Bayes risk associated with representation $W$.

Let $L$, $S$, and $O$ denote the LLM representation, the structural representation, and the one-hot representation, respectively.

Under Assumptions~\ref{ass:llm}--\ref{ass:separation}, the latent reconstruction errors satisfy
\[
\mathcal D_L \le \mathcal D_S \le \mathcal D_O.
\]
Equivalently, the induced surrogate upper bounds satisfy
\[
\mathcal U_L \le \mathcal U_S \le \mathcal U_O.
\]

In particular, the LLM representation induces the smallest surrogate upper bound, the structural representation induces a larger one, and the one-hot representation induces the largest one.

Moreover, under Assumption~\ref{ass:lipschitz}, each Bayes risk is controlled by its corresponding surrogate upper bound:
\[
\mathcal R_W \le \mathcal U_W,
\qquad W\in\{L,S,O\}.
\]
\end{proposition}

\begin{proof}
We first prove the ordering of latent reconstruction errors.

Since the LLM representation $L$ is assumed to admit a decoder for the task-relevant latent variable $Z^\star$, Lemma~\ref{lem:llm-upper} gives
\[
\mathcal D_L \le \sigma_L^2.
\]
Here, $\sigma_L^2$ is the assumed upper bound on the latent reconstruction error achievable from the LLM representation.

Since the structural representation $S$ is assumed to be task-agnostic, Assumption~\ref{ass:struct} gives
\[
\mathcal D_S \ge \rho_S^2.
\]
Here, $\rho_S^2$ is the irreducible latent uncertainty remaining after observing the structural representation.

By the task-relevant signal dominance assumption, Assumption~\ref{ass:separation}, we have
\[
\sigma_L^2 \le \rho_S^2.
\]
Combining these inequalities yields
\[
\mathcal D_L
\le
\sigma_L^2
\le
\rho_S^2
\le
\mathcal D_S.
\]
Therefore,
\[
\mathcal D_L \le \mathcal D_S.
\]

It remains to compare the structural representation $S$ with the one-hot representation $O$. By Assumption~\ref{ass:o-from-s}, there exists a measurable function $T$ such that
\[
O=T(S).
\]
Hence every $O$-measurable estimator is also $S$-measurable. In particular,
\[
\mathbb E[Z^\star\mid O]
\]
is an admissible estimator of $Z^\star$ based on $S$.

Since
\[
\mathbb E[Z^\star\mid S]
\]
is the optimal $S$-measurable estimator under squared loss, we have
\[
\mathbb E\left[
\left\|Z^\star-\mathbb E[Z^\star\mid S]\right\|^2
\right]
\le
\mathbb E\left[
\left\|Z^\star-\mathbb E[Z^\star\mid O]\right\|^2
\right].
\]
By the definition of latent reconstruction error, this is exactly
\[
\mathcal D_S \le \mathcal D_O.
\]

Combining the two comparisons gives
\[
\mathcal D_L \le \mathcal D_S \le \mathcal D_O.
\]

Now recall that
\[
\mathcal U_W := L_h^2\mathcal D_W.
\]
Since $L_h^2\ge 0$, multiplying the latent reconstruction ordering by $L_h^2$ gives
\[
\mathcal U_L \le \mathcal U_S \le \mathcal U_O.
\]

Finally, by Lemma~\ref{lem:transfer}, under the Lipschitz task-link assumption,
\[
\mathcal R_W \le L_h^2\mathcal D_W = \mathcal U_W,
\qquad W\in\{L,S,O\}.
\]
This completes the proof.
\end{proof}
\begin{remark}[Implications for Bayesian optimization]
\label{rem:bo}
The ordering of surrogate upper bounds provides a representation-level explanation for the effectiveness of Bayesian optimization.

In Bayesian optimization, the performance of the optimization procedure critically depends on the ability of the surrogate model (e.g., a Gaussian process) to accurately model the underlying objective function and provide well-calibrated uncertainty estimates. As emphasized by Snoek et al.~(2014), failures in modeling, such as violations of stationarity, can significantly degrade optimization performance, while improving the fidelity of the surrogate model can lead to faster convergence and more reliable results.

The present theorem shows that the choice of representation $W$ directly controls an upper bound on the prediction error via $\mathcal U_W = L_h^2 \mathcal D_W$. Therefore, representations that better preserve the task-relevant latent variable $Z^\star$ induce tighter surrogate upper bounds, which in turn correspond to more accurate function models.
\end{remark}
\subsection{Adaptive coarse-to-fine trade-off}
\paragraph{Proposition 3 (Adaptive scalarization induced by MGDA).}
Let
\[
\mathcal{L}_1(\mathbf{z})=\mathcal{L}_{\mathrm{rank}}(\mathbf{z}),
\qquad
\mathcal{L}_2(\mathbf{z})=\mathcal{L}_{\mathrm{graph}}(\mathbf{z}),
\]
where both objectives are differentiable with respect to $\mathbf{z}$. 
Suppose that $\mathbf{z}^\star$ is a Pareto-stationary point obtained by MGDA for the bi-objective problem
\[
\min_{\mathbf{z}}
\left(
\mathcal{L}_1(\mathbf{z}),
\mathcal{L}_2(\mathbf{z})
\right).
\]
Then there exists $\alpha^\star \in [0,1]$ such that
\begin{equation}
\alpha^\star \nabla \mathcal{L}_1(\mathbf{z}^\star)
+
(1-\alpha^\star)\nabla \mathcal{L}_2(\mathbf{z}^\star)
=0.
\end{equation}
Moreover, if $\alpha^\star>0$, then there exists an adaptive coefficient
\[
\lambda^\star=\frac{1-\alpha^\star}{\alpha^\star}
\]
such that
\begin{equation}
\nabla
\left(
\mathcal{L}_{\mathrm{rank}}(\mathbf{z}^\star)
+
\lambda^\star
\mathcal{L}_{\mathrm{graph}}(\mathbf{z}^\star)
\right)
=0.
\end{equation}
Therefore, $\mathbf{z}^\star$ is a first-order stationary point of a scalarized objective whose trade-off coefficient is determined adaptively by MGDA rather than manually specified.

\paragraph{Proof.}
MGDA finds a common descent direction by solving
\begin{equation}
\min_{\alpha \in [0,1]}
\left\|
\alpha \nabla \mathcal{L}_1(\mathbf{z})
+
(1-\alpha)\nabla \mathcal{L}_2(\mathbf{z})
\right\|_2^2 .
\end{equation}
At a Pareto-stationary point $\mathbf{z}^\star$, there exists no direction that can simultaneously decrease both objectives. This condition is equivalent to the origin belonging to the convex hull of the gradients of the objectives. Hence, for the two-objective case, there exists $\alpha^\star \in [0,1]$ such that
\begin{equation}
\alpha^\star \nabla \mathcal{L}_1(\mathbf{z}^\star)
+
(1-\alpha^\star)\nabla \mathcal{L}_2(\mathbf{z}^\star)
=0.
\end{equation}
If $\alpha^\star>0$, dividing both sides by $\alpha^\star$ yields
\begin{equation}
\nabla \mathcal{L}_1(\mathbf{z}^\star)
+
\frac{1-\alpha^\star}{\alpha^\star}
\nabla \mathcal{L}_2(\mathbf{z}^\star)
=0.
\end{equation}
By setting
\[
\lambda^\star=\frac{1-\alpha^\star}{\alpha^\star},
\]
we obtain
\begin{equation}
\nabla
\left(
\mathcal{L}_{\mathrm{rank}}(\mathbf{z}^\star)
+
\lambda^\star
\mathcal{L}_{\mathrm{graph}}(\mathbf{z}^\star)
\right)
=0.
\end{equation}
This proves the claim.
\begin{remark}[Adaptive coarse-to-fine trade-off.]
 The above result provides an interpretation of our multi-objective formulation as an adaptive trade-off between coarse-grained task-aware information and fine-grained structural consistency. 
The ranking objective $\mathcal{L}_{\mathrm{rank}}$ preserves the coarse, task-aware ordinal structure induced by the LLM, while the graph objective $\mathcal{L}_{\mathrm{graph}}$ enforces fine-grained smoothness over the structural similarity graph. 
MGDA adaptively determines the effective coefficient $\lambda^\star$ at the gradient level, rather than requiring a manually fixed weighting parameter. 
Thus, the learned representation can be viewed as a Pareto-consistent compromise between LLM-derived task-aware guidance and data-driven fine structural regularization.   
\end{remark}

\newpage
\section{More Experiments}\label{app:exp}

This section provides additional experimental details and results for TagBO. We first describe the six benchmark tasks used in our experiments. For each task, we summarize the optimization objective, search space, input variables, and output variables. We then present supplementary experimental results, including an ablation study on different feature components and a stability analysis of repeated feature extraction runs.

\subsection{Dataset}\label{app:data}
This section provides detailed descriptions of the six experiments used in this study. Each task focuses on the optimization of a particular system or material, with the goal of improving performance based on various parameters.

\paragraph{Organic Photovoltaics \citep{mmli_opv}.}
This task aims to optimize the choice of molecular building blocks for an organic photovoltaic candidate, where the goal is to maximize the \textbf{PredictedSO}, a DFT-predicted spectral overlap score. The search space is fully discrete and consists of selecting one \textbf{DB\_block} from 22 options and one \textbf{A\_block} from 100 options, yielding 2200 possible combinations. The optimization therefore focuses on identifying the pair of building blocks whose combined electronic and molecular characteristics lead to the highest predicted spectral overlap.

\begin{table}[H]
\centering
\small
\renewcommand{\arraystretch}{1.2}
\renewcommand{\tabularxcolumn}[1]{>{\raggedright\arraybackslash}m{#1}}

\begin{tabularx}{\linewidth}{|l|l|l|X|l|}
\hline
\textbf{Input/Output} & \textbf{Name} & \textbf{Range / Count} & \textbf{Description} & \textbf{Type} \\ \hline
\textbf{Inputs} & DB\_block & 22 & Donor-bridge molecular building block & Categorical \\ \hline
 & A\_block & 100 & Acceptor molecular building block& Categorical \\ \hline
\textbf{Output} & PredictedSO & $[0, +\infty)$ & DFT-predicted spectral overlap score& Continuous \\ \hline
\end{tabularx}
\vspace{1ex}
\caption{Input and Output variables for Organic Photovoltaics task.}
\end{table}

\paragraph{Suzuki--Miyaura Reaction \citep{Suzuki}.}
This task aims to maximize reaction yield on a fixed-substrate Suzuki--Miyaura reaction subset by optimizing three categorical reaction conditions: \textbf{ligand}, \textbf{base}, and \textbf{solvent}. In this subset, the electrophile, nucleophile, and catalyst are fixed, and the search space is fully discrete, consisting of 12 ligand options, 8 base options, and 4 solvent options, yielding 384 possible reaction conditions. The optimization therefore focuses on identifying the combination of ligand, base, and solvent that gives the highest \textbf{yield} for the fixed Suzuki--Miyaura coupling reaction.

\begin{table}[H]
\centering
\small
\renewcommand{\arraystretch}{1.2}
\renewcommand{\tabularxcolumn}[1]{>{\raggedright\arraybackslash}m{#1}}
\begin{tabularx}{\linewidth}{|l|l|l|X|l|}
\hline
	\textbf{Input/Output} & \textbf{Name} & \textbf{Range / Count} & \textbf{Description} & \textbf{Type} \\ \hline
	\textbf{Inputs} & ligand & 12 & Ligand selected from 12 categorical options & Categorical \\ \hline
 & base & 8 & Base selected from 8 categorical options & Categorical \\ \hline
 & solvent & 4 & Solvent selected from 4 categorical options & Categorical \\ \hline
	\textbf{Output} & yield & $[0, 100]$ & Reaction yield to be maximized & Continuous \\ \hline
\end{tabularx}
\vspace{1ex}
\caption{Input and Output variables for Suzuki--Miyaura Reaction task.}
\end{table}

\paragraph{Organic Laser Molecules \citep{dye_lasers}.}
This task aims to optimize the choice of molecular fragments for symmetric pentamer candidates designed for organic dye laser applications. The search space is fully discrete and consists of selecting one \textbf{frag\_a} from 14 options, one \textbf{frag\_b} from 13 options, and one \textbf{frag\_c} from 19 options, yielding 3458 possible molecular combinations. The optimization is multi-objective: it seeks to maximize the \textbf{peak\_score}. The task therefore focuses on identifying fragment combinations that jointly improve lasing-related photophysical performance.

\begin{table}[H]
\centering
\small
\renewcommand{\arraystretch}{1.2}
\renewcommand{\tabularxcolumn}[1]{>{\raggedright\arraybackslash}m{#1}}

\begin{tabularx}{\linewidth}{|l|l|l|X|l|}
\hline
\textbf{Input/Output} & \textbf{Name} & \textbf{Range / Count} & \textbf{Description} & \textbf{Type} \\ \hline
\textbf{Inputs} & frag\_a & 14 & Molecular fragment selected for the first fragment position & Categorical \\ \hline
 & frag\_b & 13 & Molecular fragment selected for the second fragment position & Categorical \\ \hline
 & frag\_c & 19 & Molecular fragment selected for the third fragment position & Categorical \\ \hline
\textbf{Outputs} & peak\_score & $[0, 1]$ & Photophysical peak score to be maximized & Continuous \\ \hline
\end{tabularx}

\vspace{1ex}
\caption{Input and Output variables for Organic Laser Molecules task.}
\end{table}

\paragraph{Perovskite Bandgap \citep{perovskites}.}
This task aims to optimize the composition of hybrid organic/inorganic perovskite materials, where the goal is to minimize the \textbf{hse\_gap}, the band gap computed at the HSE06 level of theory. The search space is fully discrete and consists of selecting one \textbf{organic} component from 16 options, one \textbf{cation} from 3 options, and one \textbf{anion} from 4 options, yielding 192 possible compositions. The optimization therefore focuses on identifying the combination of organic cation, metal cation, and halide anion that leads to the smallest predicted band gap, while reflecting how the physicochemical properties of the individual perovskite components influence the electronic structure of the material.

\begin{table}[H]
\centering
\small
\renewcommand{\arraystretch}{1.2}
\renewcommand{\tabularxcolumn}[1]{>{\raggedright\arraybackslash}m{#1}}

\begin{tabularx}{\linewidth}{|l|l|l|X|l|}
\hline
	\textbf{Input/Output} & \textbf{Name} & \textbf{Range / Count} & \textbf{Description} & \textbf{Type} \\ \hline
	\textbf{Inputs} & organic & 16 & Organic component& Categorical \\ \hline
 & cation & 3 & Metal cation & Categorical \\ \hline
 & anion & 4 & Halide anion & Categorical \\ \hline
	\textbf{Output} & hse\_gap & $[0, +\infty)$ & HSE06-computed band gap & Continuous \\ \hline
\end{tabularx}
\vspace{1ex}
\caption{Input and Output variables for Perovskite Bandgap task.}
\end{table}

\paragraph{Redox Flow Battery \citep{redoxmer}.}
This task aims to optimize the choice of substituent groups for redoxmer candidates used in non-aqueous redox flow batteries. The search space is fully discrete and consists of selecting one \textbf{r1\_label} from 2 options, one \textbf{r3\_label} from 8 options, one \textbf{r4\_label} from 8 options, and one \textbf{r5\_label} from 11 options, yielding 1408 possible molecular combinations. The optimization objective is to minimize the normalized average of \textbf{abs\_lam\_diff}, \textbf{ered}, and \textbf{gsol}, thereby identifying substituent combinations with improved overall redoxmer suitability for battery applications.

\begin{table}[H]
\centering
\small
\renewcommand{\arraystretch}{1.2}
\renewcommand{\tabularxcolumn}[1]{>{\raggedright\arraybackslash}m{#1}}

\begin{tabularx}{\linewidth}{|l|l|l|X|l|}
\hline
	\textbf{Input/Output} & \textbf{Name} & \textbf{Range / Count} & \textbf{Description} & \textbf{Type} \\ \hline
	\textbf{Inputs} & r1\_label & 2 & Substituent group selected for the R1 position & Categorical \\ \hline
 & r3\_label & 8 & Substituent group selected for the R3 position & Categorical \\ \hline
 & r4\_label & 8 & Substituent group selected for the R4 position & Categorical \\ \hline
 & r5\_label & 11 & Substituent group selected for the R5 position & Categorical \\ \hline
	\textbf{Outputs} & abs\_lam\_diff & $[0, +\infty)$ & Absolute wavelength difference to be minimized & Continuous \\ \hline
 & ered & $(-\infty, +\infty)$ & Reduction potential to be minimized & Continuous \\ \hline
 & gsol & $(-\infty, +\infty)$ & Solvation free energy to be minimized & Continuous \\ \hline
\end{tabularx}
\vspace{1ex}
\caption{Input and Output variables for Redox Flow Battery task.}
\end{table}

\paragraph{Crystallization Process \citep{vapdiff_crystal}.}
This task aims to optimize the conditions of a vapor diffusion crystallization experiment, where the goal is to maximize the ordinal \textbf{crystal\_score}. The search space is mixed-variable and includes the choice of \textbf{organic} cation and \textbf{solvent}, together with \textbf{concentration}, \textbf{volume}, \textbf{time}, and \textbf{temperature} settings. Specifically, the dataset contains one categorical organic component, four continuous composition and volume variables, one categorical solvent choice, and two discrete operating conditions for reaction time and reaction temperature. The task focuses on identifying feasible crystallization conditions that promote better crystal formation outcomes.

\begin{table}[H]
\centering
\small
\renewcommand{\arraystretch}{1.2}
\renewcommand{\tabularxcolumn}[1]{>{\raggedright\arraybackslash}m{#1}}
\begin{tabularx}{\linewidth}{|l|l|l|X|l|}
\hline
	\textbf{Input/Output} & \textbf{Name} & \textbf{Range / Count} & \textbf{Description} & \textbf{Type} \\ \hline
	\textbf{Inputs} & organic & 17 & Organic ammonium component& Categorical \\ \hline
 & organic\_molarity & [0, 10] & Molarity of the organic component & Continuous \\ \hline
 & solvent & 3 & Solvent& Categorical \\ \hline
 & solvent\_molarity & [0, 15] & Molarity of the solvent component & Continuous \\ \hline
 & inorganic\_molarity & [0, 5] & Molarity of the inorganic precursor & Continuous \\ \hline
 & acid\_molarity & [0, 25] & Molarity of the acid additive & Continuous \\ \hline
 & alpha\_vial\_volume & [0, 0.001] & Volume of the alpha vial mixture & Continuous \\ \hline
 & beta\_vial\_volume & [0, 0.001] & Volume of the beta vial mixture & Continuous \\ \hline
 & reaction\_time & 3 & Reaction time& Discrete \\ \hline
 & reaction\_temperature & 3 & Reaction temperature& Discrete \\ \hline
	\textbf{Output} & crystal\_score & 4 & Ordinal crystallization outcome with four ordered levels from clear solution to large crystallites & Ordinal \\ \hline
\end{tabularx}
\vspace{1ex}
\caption{Input and Output variables for Crystallization Process task.}
\end{table}
\subsection{Real-world Formulation Experiment}\label{app:real_exp}

\paragraph{Task setting.}
This task aims to maximize the measured \textbf{thermal conductivity} of a thermally conductive adhesive while satisfying mechanical and electrical performance requirements. The mixed-variable search space involves selecting 10 fillers from a library of 29 candidate materials and optimizing their 10 corresponding continuous composition variables, as summarized in Table~\ref{tab:real_exp_variables}. The thermal conductivity qualification threshold is $2.8\,\mathrm{W/(m\cdot K)}$. The formulation must also satisfy tensile strength $\geq 1\,\mathrm{MPa}$, shear strength $\geq 2\,\mathrm{MPa}$, and dielectric strength $\geq 7.9\,\mathrm{kV/mm}$. The experiment includes 34 physical evaluations, each requiring approximately three days.
\begin{table}[H]
\centering
\small
\renewcommand{\arraystretch}{1.2}
\renewcommand{\tabularxcolumn}[1]{>{\raggedright\arraybackslash}m{#1}}
\begin{tabularx}{\linewidth}{|l|l|l|X|l|}
\hline
\textbf{Input/Output} & \textbf{Name} & \textbf{Range / Count} & \textbf{Description} & \textbf{Type} \\ \hline
\textbf{Inputs} & filler\_1 & 29 & First filler & Categorical \\ \hline
 & filler\_2 & 29 & Second filler & Categorical \\ \hline
 & filler\_3 & 29 & Third filler & Categorical \\ \hline
 & filler\_4 & 29 & Fourth filler & Categorical \\ \hline
 & filler\_5 & 29 & Fifth filler & Categorical \\ \hline
 & filler\_6 & 29 & Sixth filler & Categorical \\ \hline
 & filler\_7 & 29 & Seventh filler & Categorical \\ \hline
 & filler\_8 & 29 & Eighth filler & Categorical \\ \hline
 & filler\_9 & 29 & Ninth filler & Categorical \\ \hline
 & filler\_10 & 29 & Tenth filler & Categorical \\ \hline
 & composition\_1 & -- & Amount of filler\_1 & Continuous \\ \hline
 & composition\_2 & -- & Amount of filler\_2 & Continuous \\ \hline
 & composition\_3 & -- & Amount of filler\_3 & Continuous \\ \hline
 & composition\_4 & -- & Amount of filler\_4 & Continuous \\ \hline
 & composition\_5 & -- & Amount of filler\_5 & Continuous \\ \hline
 & composition\_6 & -- & Amount of filler\_6 & Continuous \\ \hline
 & composition\_7 & -- & Amount of filler\_7 & Continuous \\ \hline
 & composition\_8 & -- & Amount of filler\_8 & Continuous \\ \hline
 & composition\_9 & -- & Amount of filler\_9 & Continuous \\ \hline
 & composition\_10 & -- & Amount of filler\_10 & Continuous \\ \hline
\textbf{Outputs} & conductivity & -- & Thermal conductivity ($\mathrm{W/(m\cdot K)}$) & Continuous \\ \hline
 & tensile & -- & Tensile strength ($\mathrm{MPa}$) & Continuous \\ \hline
 & shear & -- & Shear strength ($\mathrm{MPa}$) & Continuous \\ \hline
 & dielectric & -- & Dielectric strength ($\mathrm{kV/mm}$) & Continuous \\ \hline
\end{tabularx}
\vspace{1ex}
\caption{Input and Output variables for the Real-world Formulation Experiment.}
\label{tab:real_exp_variables}
\end{table}
\paragraph{Results.}
As shown in Fig.~\ref{fig:real_exp}, the best measured thermal conductivity increases from $0.1823\,\mathrm{W/(m\cdot K)}$ at the first evaluation to $2.995\,\mathrm{W/(m\cdot K)}$ after 34 evaluations, exceeding the qualification threshold. The final formulation achieves a tensile strength of $3.36\,\mathrm{MPa}$, a shear strength of $2.92\,\mathrm{MPa}$, and a dielectric strength of $23.5\,\mathrm{kV/mm}$, satisfying all three performance requirements.
\subsection{More Results}
\paragraph{Supplementary ablation experiment.}
For the ablation study in Fig. \ref{fig:e_ablation}, We compare the full TagBO framework with two reduced variants using only LLM-derived scoring features or only physicochemical property features across six datasets, and the trajectories show that the full method generally reaches strong objective values in fewer optimization steps while achieving comparable or better final performance in most cases.

\begin{figure}[h]
    \centering
    \includegraphics[width=1\textwidth]{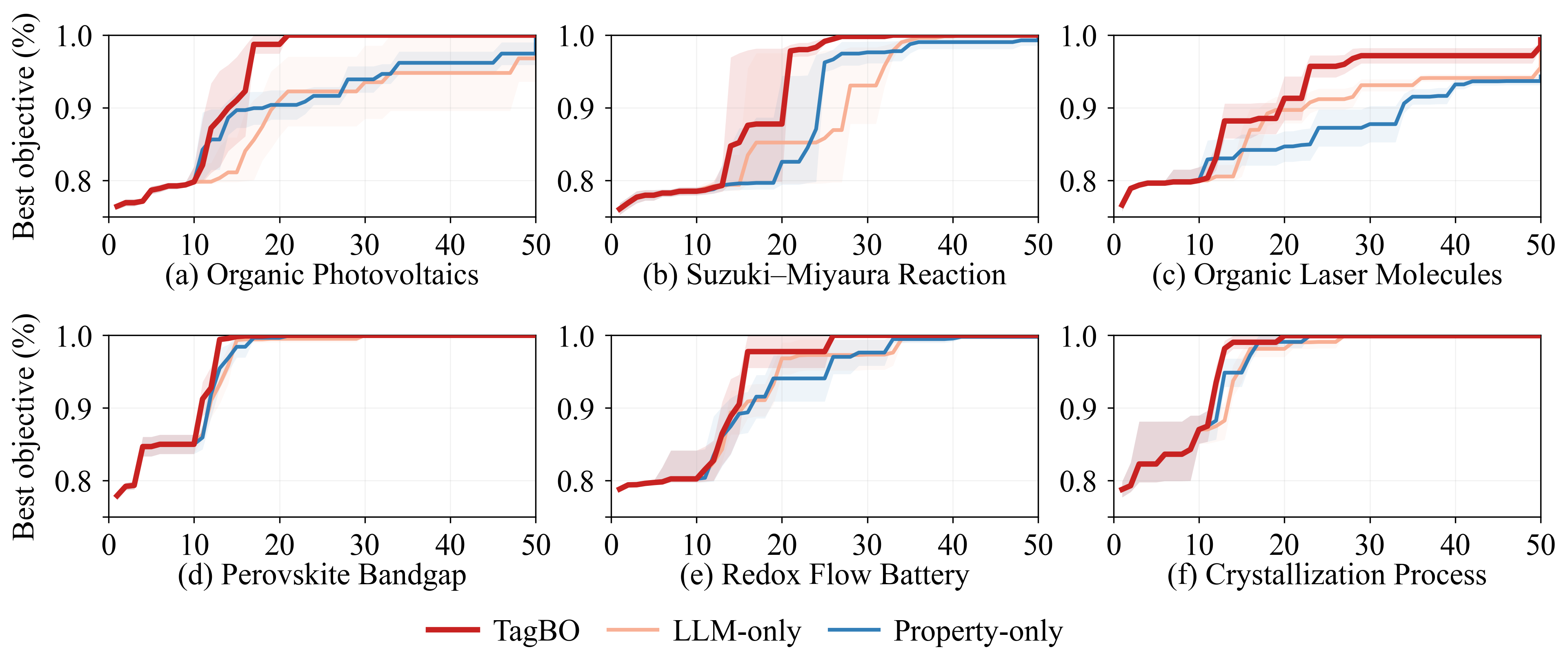}
    \caption{Ablation study of TagBO across six datasets. We compare the full TagBO framework with two reduced variants using only LLM-derived scoring features or only physicochemical property features.}
    \label{fig:e_ablation}
\end{figure}

To investigate the effect of model scale and reasoning capability, we compare Intern-S1-241B Thinking, Qwen3-235B, Qwen3-30B (Instruct and Thinking), Qwen3-8B, Qwen3-4B, Llama-3-70B, and Llama-3-3B as LLM backbones, as shown in Fig. \ref{fig:e_ablation_2}. Larger models and Thinking variants generally require fewer iterations to reach 95\% of the oracle optimum, with Intern-S1-241B Thinking achieving the best overall sample efficiency.

\begin{figure}[h]
    \centering
    \includegraphics[width=1\textwidth]{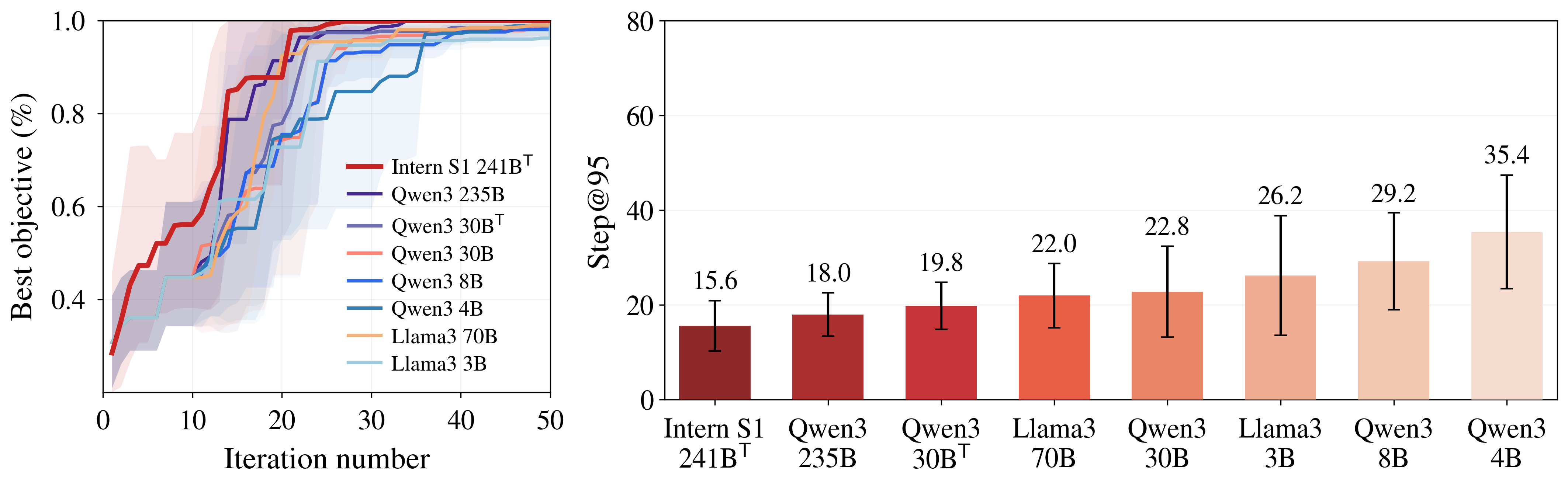}
    \caption{Optimization performance of LLM variants for feature extraction on the Suzuki–Miyaura reaction dataset. \textbf{Left}: the optimization trajectories, where curves denote the best objective value found so far, normalized by the oracle optimum. \textbf{Right}: the mean Step@95\%, where bar height indicates the average number of iterations required to reach 95\% of the optimum, and error bars reflect variation across seeds. $*^T$ denotes Thinking variants.}
    \label{fig:e_ablation_2}
\end{figure}

\paragraph{Model Stability Analysis.}
To evaluate repeatability, we perform six independent runs of Intern-S1-241B Thinking on the Suzuki–Miyaura reaction dataset. As shown in Fig.~\ref{fig:stable}, the runs show similar convergence trends, enter the high-performance regime within a narrow iteration range, and reach comparable best-so-far values. The mean Step@95 varies only slightly across runs, indicating limited variation in convergence speed. These results suggest that the proposed feature extraction pipeline is stable under repeated runs.

\begin{figure}[h]
    \centering
    \includegraphics[width=1\textwidth]{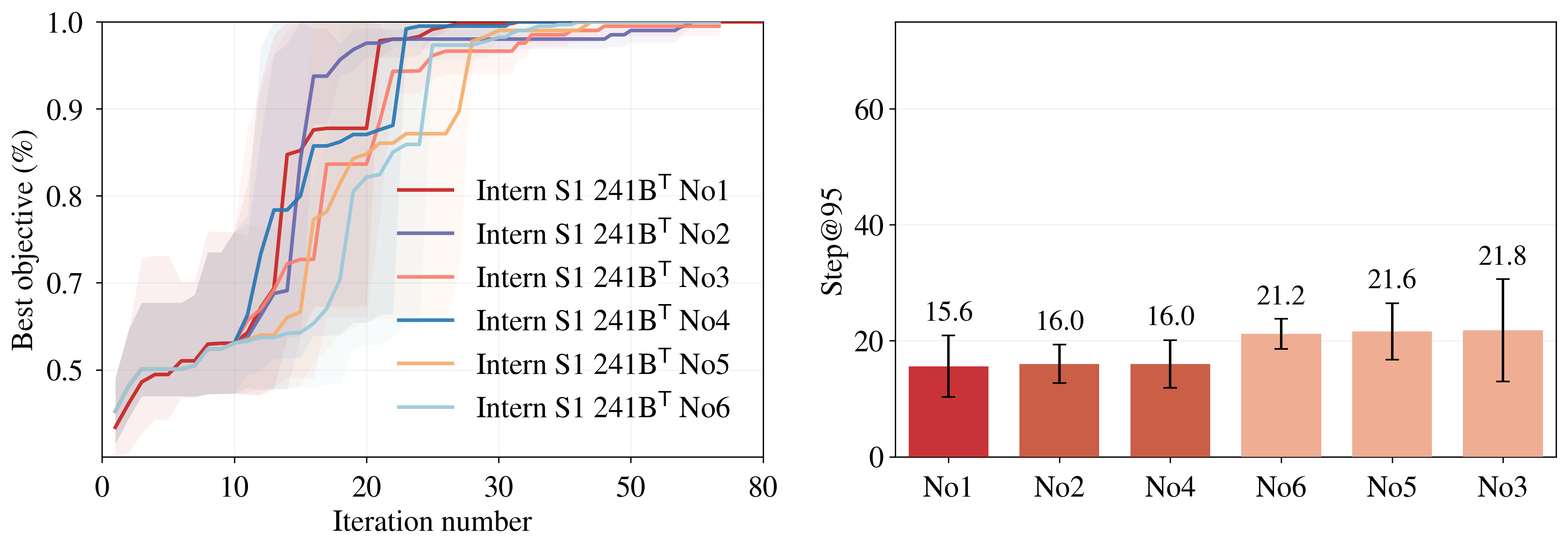}
  \caption{Optimization performance of repeated Intern-S1-241B Thinking variants for feature extraction on the Suzuki–Miyaura reaction dataset. \textbf{Left}: the optimization trajectories, where curves denote the best objective value found so far, normalized by the oracle optimum. \textbf{Right}: the mean Step@95\%, where bar height indicates the average number of iterations required to reach 95\% of the optimum, and error bars reflect variation across seeds. $*^T$ denote Thinking variants}
  \label{fig:stable}
\end{figure}

\section{Limitations}
The effectiveness of TagBO depends on the quality of the task-aware relational priors extracted from the LLM. In domains where relevant scientific knowledge is sparse, highly specialized, or poorly represented in the pretraining corpus, the induced ordinal relations may become noisy or misaligned with the downstream optimization objective. We believe that retrieval-augmented generation (RAG) and domain-specific knowledge integration may help alleviate this limitation by grounding relational reasoning in external scientific resources, and we plan to explore this direction in future work.

In addition, TagBO primarily modifies the representation and surrogate prior of Bayesian optimization while leaving the underlying BO optimization procedure unchanged. Consequently, our theoretical analysis focuses on representation-level surrogate behavior rather than cumulative regret guarantees for the overall sequential optimization process. Deriving regret bounds that explicitly characterize the interaction between task-aware representations, kernel geometry, and acquisition dynamics remains an important direction for future research.

\section*{LLM Usage Statement}

Large language models (LLMs) are core components of the proposed method. Specifically. But, LLMs were not used to generate experimental results automatically without human verification. All reported analyses, evaluations, and conclusions were verified by the authors.
%%%%%%%%%%%%%%%%%%%%%%%%%%%%%%%%%%%%%%%%%%%%%%%%%%%%%%%%%%%%

\newpage
\section*{NeurIPS Paper Checklist}

%%% BEGIN INSTRUCTIONS %%%

%%% END INSTRUCTIONS %%%

\begin{enumerate}

\item {\bf Claims}
    \item[] Question: Do the main claims made in the abstract and introduction accurately reflect the paper's contributions and scope?
    \item[] Answer: \answerYes{} % Replace by \answerYes{}, \answerNo{}, or \answerNA{}.
    \item[] Justification: The abstract and introduction accurately summarize the main contributions and scope of the paper.

\item {\bf Limitations}
    \item[] Question: Does the paper discuss the limitations of the work performed by the authors?
    \item[] Answer: \answerYes{} % Replace by \answerYes{}, \answerNo{}, or \answerNA{}.
    \item[] Justification: The limitations is discussed in Appendix.

\item {\bf Theory assumptions and proofs}
    \item[] Question: For each theoretical result, does the paper provide the full set of assumptions and a complete (and correct) proof?
    \item[] Answer: \answerYes{} % Replace by \answerYes{}, \answerNo{}, or \answerNA{}.
    \item[] Justification: Proof is provided in Appendix.

    \item {\bf Experimental result reproducibility}
    \item[] Question: Does the paper fully disclose all the information needed to reproduce the main experimental results of the paper to the extent that it affects the main claims and/or conclusions of the paper (regardless of whether the code and data are provided or not)?
    \item[] Answer: \answerYes{}% Replace by \answerYes{}, \answerNo{}, or \answerNA{}.
    \item[] Justification: We provide code in supplemental materials.

\item {\bf Open access to data and code}
    \item[] Question: Does the paper provide open access to the data and code, with sufficient instructions to faithfully reproduce the main experimental results, as described in supplemental material?
    \item[] Answer:  \answerYes{} % Replace by \answerYes{}, \answerNo{}, or \answerNA{}.
    \item[] Justification:  We provide code in supplemental materials.

\item {\bf Experimental setting/details}
    \item[] Question: Does the paper specify all the training and test details (e.g., data splits, hyperparameters, how they were chosen, type of optimizer) necessary to understand the results?
    \item[] Answer: \answerYes{} % Replace by \answerYes{}, \answerNo{}, or \answerNA{}.
    \item[] Justification: We provide code in supplemental materials.

\item {\bf Experiment statistical significance}
    \item[] Question: Does the paper report error bars suitably and correctly defined or other appropriate information about the statistical significance of the experiments?
    \item[] Answer: \answerYes{} % Replace by \answerYes{}, \answerNo{}, or \answerNA{}.
    \item[] Justification: Variance across runs is reported using shaded regions in Fig. 2 and Fig. 4(a,b), while box plots are used to illustrate the variability in the ablation studies.

\item {\bf Experiments compute resources}
    \item[] Question: For each experiment, does the paper provide sufficient information on the computer resources (type of compute workers, memory, time of execution) needed to reproduce the experiments?
    \item[] Answer: \answerYes{} % Replace by \answerYes{}, \answerNo{}, or \answerNA{}.
    \item[] Justification: Variance across runs is reported using shaded regions in Fig. 2 and Fig. 4(a,b), while box plots are used to illustrate the variability in the ablation studies.
    
\item {\bf Code of ethics}
    \item[] Question: Does the research conducted in the paper conform, in every respect, with the NeurIPS Code of Ethics \url{https://neurips.cc/public/EthicsGuidelines}?
    \item[] Answer: \answerYes{} % Replace by \answerYes{}, \answerNo{}, or \answerNA{}.
    \item[] Justification: The research complies with the NeurIPS Code of Ethics.

\item {\bf Broader impacts}
    \item[] Question: Does the paper discuss both potential positive societal impacts and negative societal impacts of the work performed?
    \item[] Answer: \answerNA{} % Replace by \answerYes{}, \answerNo{}, or \answerNA{}.
    \item[] Justification: There is no societal impact of the work performed.
    
\item {\bf Safeguards}
    \item[] Question: Does the paper describe safeguards that have been put in place for responsible release of data or models that have a high risk for misuse (e.g., pre-trained language models, image generators, or scraped datasets)?
    \item[] Answer: \answerNA{} % Replace by \answerYes{}, \answerNo{}, or \answerNA{}.
    \item[] Justification: The proposed method does not introduce high-risk generative models or sensitive datasets requiring special safeguards.

\item {\bf Licenses for existing assets}
    \item[] Question: Are the creators or original owners of assets (e.g., code, data, models), used in the paper, properly credited and are the license and terms of use explicitly mentioned and properly respected?
    \item[] Answer: \answerYes{} % Replace by \answerYes{}, \answerNo{}, or \answerNA{}.
    \item[] Justification: All datasets and external codebases used in this work are properly cited, and their licenses and usage terms are respected.

\item {\bf New assets}
    \item[] Question: Are new assets introduced in the paper well documented and is the documentation provided alongside the assets?
    \item[] Answer:\answerNA{} % Replace by \answerYes{}, \answerNo{}, or \answerNA{}.
    \item[] Justification: This paper does not introduce new public assets such as datasets or models.

\item {\bf Crowdsourcing and research with human subjects}
    \item[] Question: For crowdsourcing experiments and research with human subjects, does the paper include the full text of instructions given to participants and screenshots, if applicable, as well as details about compensation (if any)? 
    \item[] Answer: \answerNA{} % Replace by \answerYes{}, \answerNo{}, or \answerNA{}.
    \item[] Justification: This work does not involve crowdsourcing or human subject experiments.

\item {\bf Institutional review board (IRB) approvals or equivalent for research with human subjects}
    \item[] Question: Does the paper describe potential risks incurred by study participants, whether such risks were disclosed to the subjects, and whether Institutional Review Board (IRB) approvals (or an equivalent approval/review based on the requirements of your country or institution) were obtained?
    \item[] Answer: \answerNA{}% Replace by \answerYes{}, \answerNo{}, or \answerNA{}.
    \item[] Justification: This work does not involve human subject research requiring IRB approval.

\item {\bf Declaration of LLM usage}
    \item[] Question: Does the paper describe the usage of LLMs if it is an important, original, or non-standard component of the core methods in this research? Note that if the LLM is used only for writing, editing, or formatting purposes and does \emph{not} impact the core methodology, scientific rigor, or originality of the research, declaration is not required.
    %this research? 
    \item[] Answer: \answerYes{} % Replace by \answerYes{}, \answerNo{}, or \answerNA{}.
    \item[] Justification: LLMs are core components of the proposed framework

\end{enumerate}

\end{document}